\documentclass{article}

\usepackage{microtype}
\usepackage{graphicx}
\usepackage{subfigure}
\usepackage{booktabs} 

\usepackage{hyperref}

\usepackage[accepted]{icml2024}

\usepackage{amsmath}
\usepackage{amssymb}
\usepackage{mathtools}
\usepackage{amsthm}
\usepackage{multirow} 

\usepackage{amsmath,amsfonts,bm}

\def\eqref#1{equation~\ref{#1}}

\def\1{\bm{1}}

\def\R{{R}}

\def\va{{\bm{a}}}
\def\vb{{\bm{b}}}
\def\vc{{\bm{c}}}

\def\ve{{\bm{e}}}

\def\vg{{\bm{g}}}
\def\vh{{\bm{h}}}

\def\vm{{\bm{m}}}

\def\vo{{\bm{o}}}
\def\vp{{\bm{p}}}

\def\vs{{\bm{s}}}
\def\vt{{\bm{t}}}
\def\vu{{\bm{u}}}
\def\vv{{\bm{v}}}

\def\vx{{\bm{x}}}

\def\mI{{\bm{I}}}

\def\mM{{\bm{M}}}

\def\mO{{\bm{O}}}

\def\mW{{\bm{W}}}

\def\mZ{{\bm{Z}}}

\DeclareMathAlphabet{\mathsfit}{\encodingdefault}{\sfdefault}{m}{sl}
\SetMathAlphabet{\mathsfit}{bold}{\encodingdefault}{\sfdefault}{bx}{n}

\def\gA{{\mathcal{A}}}

\def\gE{{\mathcal{E}}}

\def\gG{{\mathcal{G}}}

\def\gM{{\mathcal{M}}}
\def\gN{{\mathcal{N}}}

\def\gS{{\mathcal{S}}}

\def\gV{{\mathcal{V}}}

\def\sR{{\mathbb{R}}}

\DeclareMathOperator*{\argmin}{arg\,min}

\usepackage[capitalize,noabbrev]{cleveref}

\theoremstyle{plain}
\newtheorem{theorem}{Theorem}[section]

\newtheorem{corollary}[theorem]{Corollary}
\theoremstyle{definition}
\newtheorem{definition}[theorem]{Definition}

\theoremstyle{remark}

\usepackage[textsize=tiny]{todonotes}

\icmltitlerunning{Subequivariant Reinforcement Learning in 3D Multi-Entity Physical Environments}

\begin{document}

\twocolumn[
\icmltitle{Subequivariant Reinforcement Learning\\in 3D Multi-Entity Physical Environments}



\icmlsetsymbol{equal}{*}

\begin{icmlauthorlist}
\icmlauthor{Runfa Chen}{equal,thu}
\icmlauthor{Ling Wang}{equal,xian}
\icmlauthor{Yu Du}{naval,thu}
\icmlauthor{Tianrui Xue}{nyu}
\icmlauthor{Fuchun Sun}{thu,bosch}
\icmlauthor{Jianwei Zhang}{hbu}
\icmlauthor{Wenbing Huang}{ruc,bdmam}
\end{icmlauthorlist}

\icmlaffiliation{thu}{Dept. of Comp. Sci. \& Tech. , Institute for AI, BNRist Center, Tsinghua University}
\icmlaffiliation{bosch}{THU-Bosch JCML Center} 
\icmlaffiliation{xian}{Dept. of Info. Eng., Xi’an Research Institute of High-Tech}
\icmlaffiliation{naval}{School of Elec. Eng., Naval University of Engineering}
\icmlaffiliation{ruc}{Gaoling School of Artificial Intelligence, Renmin University of China} 
\icmlaffiliation{bdmam}{Beijing Key Laboratory of Big Data Management and Analysis Methods}
\icmlaffiliation{nyu}{College of Arts \& Sci., New York Unviersity}  
\icmlaffiliation{hbu}{TAMS, Dept. of Informatics, University of Hamburg} 

\icmlcorrespondingauthor{Fuchun Sun}{fcsun@mail.tsinghua.edu.cn}

\icmlkeywords{Machine Learning, ICML}

\vskip 0.3in
]



\printAffiliationsAndNotice{\icmlEqualContribution} 

\begin{abstract}
Learning policies for multi-entity systems in 3D environments is far more complicated against single-entity scenarios, due to the exponential expansion of the global state space as the number of entities increases.  
One potential solution of alleviating the exponential complexity is dividing the global space into independent local views that are invariant to transformations including translations and rotations. 
To this end, this paper proposes \emph{Subequivariant Hierarchical Neural Networks} (SHNN) to facilitate multi-entity policy learning. In particular, SHNN first dynamically decouples the global space into local entity-level graphs via task assignment. 
Second, it leverages subequivariant message passing over the local entity-level graphs to devise local reference frames, remarkably compressing the representation redundancy, particularly in gravity-affected environments.
Furthermore, to overcome the limitations of existing benchmarks in capturing the subtleties of multi-entity systems under the Euclidean symmetry, we propose the \emph{Multi-entity Benchmark} (\textsc{MeBen}), a new suite of environments tailored for exploring a wide range of multi-entity reinforcement learning. 
Extensive experiments demonstrate significant advancements of SHNN  on the proposed benchmarks compared to existing methods. 
Comprehensive ablations are conducted to verify the indispensability of task assignment and subequivariance.
\end{abstract}

\begin{figure}[t!]
\centering
\includegraphics[width=\linewidth]{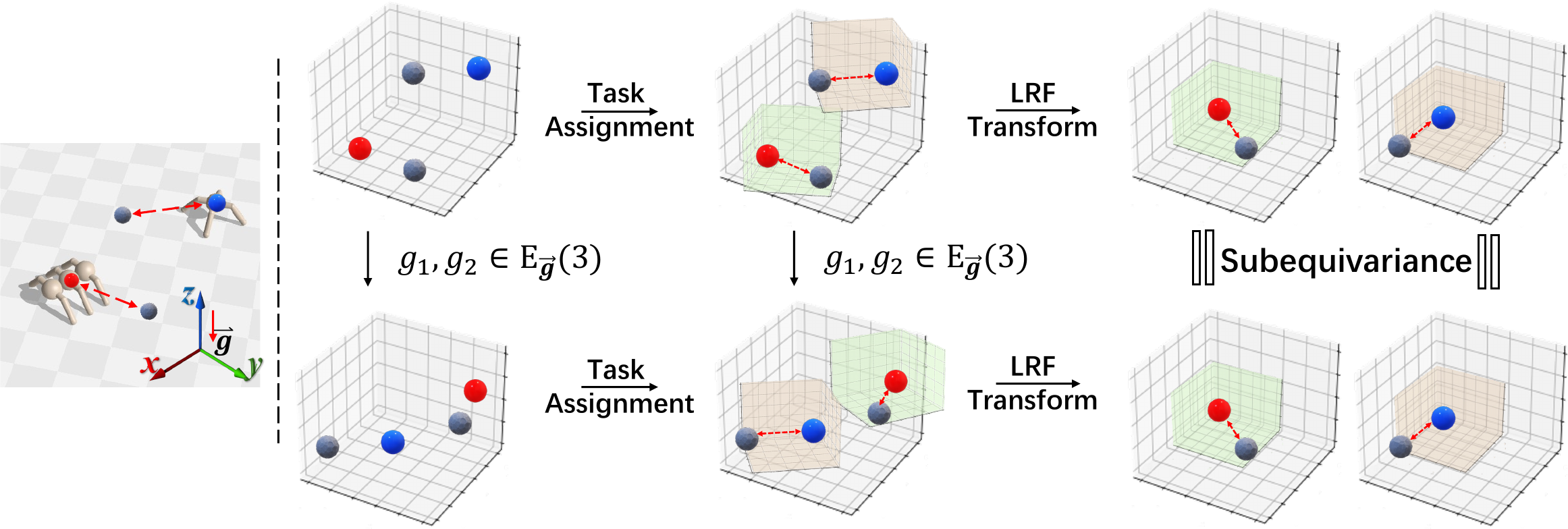}
\vspace{-.2in}
\caption{
Illustration of the symmetry in our 3D multi-entity physical environments. In this example, there are two agents (red and blue) navigating towards two objects (grey). To mitigate the exponential-growth complexity, we conduct task assignment to decouple the whole state space into local views (the orange and green transparent coordinate frames), where one agent is assigned for one object. These local views can be represented by local reference frame (LRF), leading to representations that are independent of any translation, or rotation around the gravity direction of the global coordinates of the entities. Such symmetry is called $\text{E}(3)$ subequivariance, distinct from conventional $\text{E}(3)$ equivariance, accounting for gravitational effects. In form, subequivariance is encapsulated by the group $\text{E}_{\vec{\boldsymbol{g}}}(3)$—translations/rotations/reflections along gravity $\vec{\boldsymbol{g}}$ (a 3D Euclidean subgroup of $\text{E}(3)$ in \cref{sec:symmetry}).
Codes are available on our project page: \href{https://alpc91.github.io/SMERL/}{https://alpc91.github.io/SMERL/}.
}
\label{fig:teaser}
\vspace{-.13in}
\end{figure}

\section{Introduction}
Learning to navigate, control, cooperate, and compete in the 3D physical world is a fundamental task in developing intelligent agents.
Deep reinforcement learning (RL) has made impressive breakthroughs, particularly in single-entity systems, with agent policies evolving through environmental interactions~\cite{mnih2015human,silver2016mastering,mnih2016asynchronous,schulman2017proximal,bansal2018emergent,liu2018emergent,liu2022motor}.
However, an intricate challenge is generalizing across configurations like transformations, morphologies, and tasks, which are interlinked and complicate the learning process.
In particular, multi-entity systems, which include agents, objects, and other entities defined in~\cite{spelke2022babies}, present considerable challenges compared to single-entity scenarios, partly due to exponential expansion of global transformations as the number of entities increases~\cite{deng2023banana}. 
In~\cref{fig:teaser}, for example, any horizontal rotation of the entities, although producing different global coordinates and representations, does not change the essential geometry and their local views of different entities. Such symmetry, defined as \emph{subequivariance} (formal definition is given in \cref{sec:symmetry}) in this paper, provides a potential way to reduce the complexity of the state space.
There are certain previous studies leveraging this type of symmetry in RL.
For instance, the methods based on heading normalization (HN)~\cite{won2020scalable,won2022physics} transform global coordinates into a local reference frame (LRF), which is yet non-learnable and non-adjustable with respect to the goal of the task.  Morphology-based RL advances~\cite{chen2023sgrl} for single entity have incorporated subequivariance~\cite{han2022learning} into policy modeling, reducing reliance on hand-crafted LRFs. However, extending subequivariance from single-entity to inter-entity transformations reveals unexplored challenges due to the coupled local space of each entity.

To tackle the difficulties of interdependence and generalization,
we introduce a novel framework, named \emph{Subequivariant Hierarchical Neural Networks} (SHNN), integrating task assignment and local entity-level subequivariant message passing within a hierarchical network architecture. SHNN offers two key advancements:
\textbf{1.} We implement the task assignment using bipartite graph matching to dynamically construct local entity-level graphs. This approach aids in managing interdependence among entities by decoupling local transformations from the overall structure.
\textbf{2.} We implement local entity-level subequivariant message passing, effectively compiling information from related entities. To guide body-level control, we utilize entity-level information to define a LRF for each entity, effectively compressing the global state space and facilitating the generalization of body-level policy in a lossless way.

Moreover, another reason for this limited exploration is the absence of suitable environments in existing morphology-based~\cite{huang2020one, chen2023sgrl,furuta2023asystem} and multi-agent reinforcement learning (MARL)~\cite{samvelyan2019starcraft,de2020deep,ellis2022smacv2,flair2023jaxmarl,lechner2023gigastep} frameworks. 
These benchmarks inadequately probe complex entity interactions, especially in scenarios involving multi-agent dynamics under a diverse range of inter-entity transformations.
To bridge this gap, we present a new suite of \emph{Multi-entity Benchmark} (\textsc{MeBen}) in 3D space. Built upon JAX-based RL environments~\cite{jax2018github,jraph2020github,flax2020github,brax2021github,gu2021braxlines}, \textsc{MeBen} is designed to investigate multi-entity interactions, encompassing both cooperative and competitive dynamics, within physical geometric symmetry constraints that include a diverse range of inter-entity transformations.

Our contributions are summarized as follows:
\begin{itemize}
    \item To effectively optimize the policy in 3D multi-entity physical environments, we propose SHNN, a novel framework that offers a superior plug-in alternative to hand-crafted LRFs. It decouples local transformations from the overall structure and compresses the state space by leveraging local physical geometric symmetry, particularly in gravity-affected environments.
    
    \item We introduce \textsc{MeBen}, a collection of subequivariant morphology-based MARL environments, designed for in-depth exploration of multi-entity interactions within physical geometric symmetry constraints. These environments, including a diverse range of inter-entity transformations, facilitate both cooperative and competitive dynamics.

    \item We demonstrate the effectiveness of SHNN in the proposed 3D multi-entity physical environments, including Team Reach and Team Sumo \footnote{For detailed task settings, see \cref{sec:bench_construct}.}. Our extensive ablations and comparative analyses also reveal the efficacy of the proposed ideas.
\end{itemize}

\section{Preliminaries}

\paragraph{Geometric Symmetry}
\label{sec:symmetry}
The symmetrical structure in 3D  environments is $\text{E}(3)$, which is a 3-dimensional Euclidean \emph{group}~\cite{dresselhaus2007group} that consists of rotations, reflections, and translations. 
\begin{definition}[Group]
\label{def:group}
A group $G$ is a set of transformations with a binary operation ``$\cdot$'' satisfying these properties: ``$\cdot$'' is closed under associative composition, there exists an identity element, and each element must have an inverse. 
\end{definition}
\vspace{-.13in}
Symmetrical structure enforced on the model~\cite{worrall2017harmonic,van2020mdp,thomas2018tensor,fuchs2020se3,jing2020gvp,deng2021vector,villar2021scalars,satorras2021en,huang2022equivariant,han2022learning,luo2022equivariant,chen2023sgrl,joshi2023expressive,wang2024equivariant,han2024survey} is formally described by the concept of \emph{equivariance}.
\begin{definition}[Equivariance]
    \label{def:equ}
    Suppose $\Vec{\mZ}$ \footnote{Note that for the input of $f$, we have added the right-arrow superscript on $\Vec{\mZ}$ to distinguish it from the scalar $\vh$ that is unaffected by the transformations.} to be 3D geometric vectors (positions, velocities, etc) that are steerable by a group $G$, and $\vh$  non-steerable features.
    The function $f$ is $G$-equivariant, if for any transformation $g\in G$, $f(g\cdot\Vec{\mZ},\vh)=g\cdot f(\Vec{\mZ},\vh)$, $\forall\Vec{\mZ}\in\sR^{3\times m}, \vh\in\sR^{d}$.  Similarly, $f$ is invariant if $f(g\cdot\Vec{\mZ},\vh)= f(\Vec{\mZ},\vh)$. 
\end{definition}
\vspace{-.13in}
Specifically, the E(3) operation ``$\cdot$'' is instantiated as $g\cdot\Vec{\mZ}\coloneqq \mO\Vec{\mZ}$ for the orthogonal group that consists of rotations and reflections where $\mO\in \text{O}(3)\coloneqq\{\mO\in\sR^{3\times 3}|\mO^\top\mO=\mI\}$, and is additionally implemented as the translation $g\cdot \vec{\vx} \coloneqq \vec{\vx} + \vec{\vt}$ for the 3D coordinate vector where $\vec{\vt}\in \text{T}(3) \coloneqq\{\vec{\vt}\in\R^{3}\}$.
To align with the principles of classical physics under the influence of gravity, we introduce a relaxation of the group constraint. Particularly, we consider equivariance within the subgroup of $\text{E}(3)$ induced by gravity $\vec{\vg}\in\sR^3$, defined as $\text{O}_{\vec\vg}(3)\coloneqq\{\mO\in\sR^{3\times 3}|\mO^\top\mO=\mI, \mO\vec\vg=\vec\vg \}$ and $\text{T}_{\vec\vg}(3) \coloneqq\{\vec{\vt}\in\R^{3}|\vec{\vt}\vec\vg=\vec{\bm{0}}\}$. By this means, the $\text{E}_{\vec{\bm{g}}}(3)$-equivariance is only restrained to the translations/rotations/reflections along the direction of $\vec\vg$. We term \emph{subequivariance} primarily refers to $\text{E}_{\vec{\bm{g}}}(3)$-equivariance.

\vspace{-.13in}
\paragraph{Problem Definition and Notation}
Our investigation in 3D physical environments delves into the interactions among multiple entities ~\cite{spelke2022babies}, including agents with distinct morphologies that enable sophisticated form and motion control, objects that critically influence the system's symmetry and dynamics.
We model this intricate setting as a Decentralized Partially Observable Markov Decision Process (Dec-POMDP) ~\cite{bernstein2002complexity, oliehoek2016concise}, represented by \( G = \langle \gS, \gA, P, O, r, \gamma, \Omega \rangle \), where \( \gS \) indicates the state space, \( \gA \) the action space, \( P \) the transition probability, \( O \) the observation function, \( r \) the reward function, and \( \gamma \) the discount factor.
The complete set of all entities is denoted as \( \Omega = \{1, \cdots, N+M\} \), where \( N \) denotes the number of cooperative agents, while \( M \) represents the count of objects (or competitive agents). 
\textbf{1. Observation.} Each agent, denoted as \( i \), is composed of \( K_i \) bodies.
At any given timestep \( t \), agent \( i \) obtains a unique observation \( \vo_{i}(t) := O(\vs(t), i) \) from the global state \( \vs(t) \in \gS \), capturing its individual perspective of the environment. This observation encompasses detailed state information about agent \( i \) itself, represented as \( \{\vs_{i,k}(t)\}_{k=1}^{K_i} \). However, 
an agent's awareness of other entities is confined to states of their root bodies, denoted as \( \vs_{j,1}(t) \) for each \( j \in \Omega \setminus \{i\} \). Thus, the complete observation for agent \( i \) at time \( t \) is formulated as \( \vo_{i}(t) := \{\{\vs_{i,k}(t)\}_{k=1}^{K_i}, \vs_{j,1}(t) \mid j \in \Omega \setminus \{i\}\} \).
\textbf{2. Action and Reward.} The decision-making process for each agent \( i \) at timestep \( t \) involves selecting an action \( \va_i(t) :=\{a_{i,k}(t)\}_{k=2}^{K_i} \in \gA \) based on its policy \( \pi_{\theta_i}(\va_i(t) | \vo_i(t)) \). Each actuator \( k \) of the agent, ranging from 2 to \( K_i \), contributes by generating a torque \( a_{i,k}(t) \in [-1,1] \). Consequently, the aggregated action \( \va(t) = (\va_1(t), \cdots, \va_N(t)) \in \gA^N \) arises from the combined actions of all agents. 
The environment reacts by transitioning to a new global state \( \vs(t+1) \sim P(\vs(t+1) | \vs(t), \va(t)) \) and allocates a shared team reward \( r(\vs(t), \va(t)) \).
\textbf{3. Actor and Critic.} 
Multiple entities interactions, depending on the number of agents and their relational dynamics, are typically categorized into single-agent, cooperative, competitive, and mixed interactions. To effectively navigate this complex multi-agent environment, Multi-Agent Proximal Policy Optimization (MAPPO) ~\cite{schulman2017proximal,yu2022surprising} is employed to optimize a joint policy $\bm{\pi}_{\theta}=(\pi_{\theta_1},\cdots,\pi_{\theta_N})$ in this intricate environment. Our objective is to maximize the expected global return:
\begin{equation}
\label{eq:objective}
\mathcal{J}(\theta)=\mathbb{E}_{\bm{\pi}_{\theta}} \sum_{t=0}^\infty\left[\gamma^t r(\vs(t), \va(t))\right].
\end{equation}
MAPPO involves the development of both the joint policy \( \pi_{\theta} \) and a value function $V_{\phi}$, crucial for variance reduction and integrating information beyond the agents' local observations, adhering to the Centralized Training with Decentralized Execution (CTDE) paradigm~\cite{lowe2017multi,sunehag2018value,rashid2018qmix,foerster2018counterfactual,kuba2021trust,wen2022multi,yu2022surprising,jeon2022maser}.

\vspace{-.1in}
\section{Subequivariant Hierarchical Neural Networks}
\label{sec:SHNN}

In this section, we introduce our entire framework \emph{Subequivariant Hierarchical Neural Networks} (SHNN), visualized in \cref{fig:framework}, consisting of an input processing module, a novel task assignment to decouple local transformations from the overall structure, a local entity-level subequivariant message passing for expressive information passing and fusion, a local reference frame transform to addresses local transformations by leveraging local physical geometric symmetry in environments with gravity and a body-level control module to obtain the final policy.
Building upon these, the global state space is effectively compressed.

\vspace{-.1in}
\paragraph{Input Processing} 
The observation input for agent \( i \) at time \( t \) is formulated as \( \vo_{i}(t) = \{\{\vs_{i,k}(t)\}_{k=1}^{K_i}, \vs_{j,1}(t) \mid j \in \Omega \setminus \{i\}\} \)\footnote{For simplicity, we omit the index $t$ henceforth in the above notations at time $t$.}.
To adhere to the constraints of physical geometric symmetry, the state \( \vs_{i,k} \) is subdivided into directional geometric vectors \( \vec{\mZ}_{i,k} \) and scalar features \( \vh_{i,k} \). Elements in \( \vec{\mZ}_{i,k} \) will rotate according to the transformation \( g \in \text{O}_{\vec\vg}(3) \), while those in \( \vh_{i,k} \) remain unchanged. Specifically, in our 3D environment, \( \vec{\mZ}_{i,k} \in \R^{3 \times 3} \) comprises the position \( \vec{\vp}_{i,k} \in \sR^3 \), positional velocity \( \vec{\vv}_{i,k} \in \sR^3 \), and rotational velocity \( \vec{\bm{\omega}}_{i,k} \in \sR^3 \).
Here, \( \vec{\vp}_{i,k} \) is transformed into a translation-invariant representation by redefining it as \( \vec{\vp}_{i,k} - \vec{\vc} \), where \( \vec{\vc} = \frac{1}{N+M} \sum_{i=1}^{N+M}\vec{\vp}_{i,1} \). This operation subtracts \( \vec{\vc} \), the average root position of all entities, thereby ensuring translation invariance.
The scalar features \( \vh_{i,k} \in \sR^{13} \) include the rotation angles \( \kappa_{i,k}, \zeta_{i,k}, \delta_{i,k} \) of the joint axes and their corresponding ranges, along with a 4-dimensional one-hot vector indicating the type of body such as ``torso", ``limb", or ``ball". The direction of gravity $\vec\vg$ is set to be along the $z$-axis.

\begin{figure*}[t!]
\centering
\includegraphics[width=\linewidth]{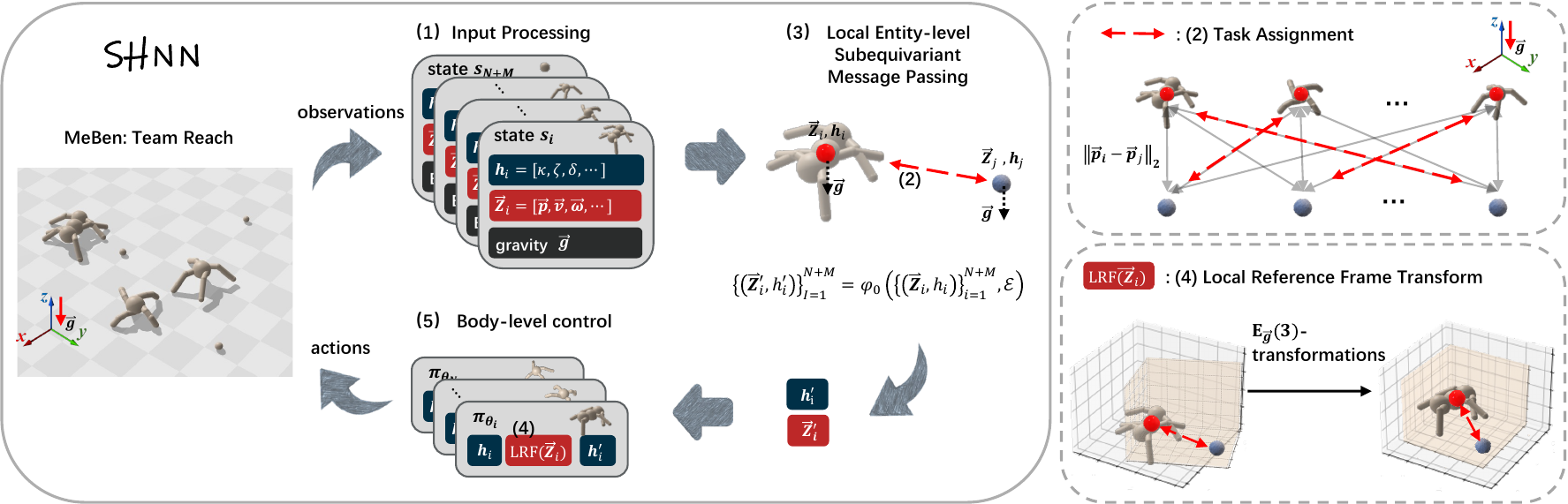}
\vspace{-.25in}
\caption{The flowchart of SHNN. On the left, The states of each entity $i$ are processed into scalar features $\vh_{i}$ and directional geometric vectors $\vec{\mZ}_{i}$, and are updated by local entity-level subequivariant message passing in a task assignment entity-level graph.  Finally, the invariant local reference frame body-level control policy is obtained. Here, function $\varphi_o$ serves as the local entity-level subequivariant MP, and $\gE$ is the local entity-level graph topology.
The right side illustrates our key innovation: the dynamic task assignment leveraging bipartite graph matching, and the construction of an $\text{E}_{\vec{\bm{g}}}(3)$-equivariant local reference frame for each entity to address local transformations.}
\label{fig:framework}
\vspace{-.13in}
\end{figure*}

\vspace{-.1in}
\paragraph{Task Assignment}  
In multi-entity environments, it is crucial to manage complex interactions among entities.
Initially, we consider a fully connected entity-level graph \( \gG = (\gV, \gE) \), with \( \gV = \Omega \) and \( \gE = \{(i, j): i, j \in \Omega, i \neq j\} \). 
To decouple local transformations from the overall structure,
we introduce task assignment 
which dynamically adjusts the edges, forming local graphs of associated entities. The graph is then redefined as a task assignment entity-level graph, where \( \gE = \{(i, j): i, j \in \Omega, \mathcal{C}(i) = \mathcal{C}(j)\} \) and \(\mathcal{C}\) are the assignment labels. This assignment, akin to part segmentation~\cite{deng2023banana}, selectively links related entities, like agents and their objects.
Specifically, in this study, we employ a rule-based task assignment approach using bipartite graph matching, guided by inter-entity distance costs.
Due to the static programming characteristics of JAX, implementing the Hungarian matching algorithm proves challenging. 
Consequently, we employ a greedy strategy-based bipartite graph matching (refer to \cref{alg:greedy_bipartite_ocp}).

\vspace{-.1in}
\paragraph{Local Entity-level Subequivariant Message Passing}
We propose enhancing mainstream Neural Networks (such as MLPs) in RL with an additional local entity-level subequivariant message passing (MP). This task assignment entity-level graph adeptly integrates local entity-level information from related entities and, through local subequivariant MP, efficiently instills the desired local physical geometric symmetry in environments with gravity into mainstream Neural Networks to address local transformations.

For each entity \( i \in \Omega \), input node features are initialized using the entity's root state. Specifically, \( \vec{\mZ}_i \) is assigned as \( \vec{\mZ}_{i,1} \), and \( \vh_i \) is set as \( [\vh_{i,1}, \vec{\vp}_{i,1}^z] \), where [ ] is the stack along the last dimension and \( \vec{\vp}_{i,1}^z \) represents the projection of the coordinate \( \vec{\vp}_{i,1} \) onto the \( z \)-axis. This projection effectively indicates the relative height of entity \( i \), considering the ground as the reference point.

Within the context of entity-level interactions, our function $\varphi_o$ serves as the local entity-level subequivariant MP. It updates each entity's node features by considering the collective input features and the established graph connectivity.
\begin{align}
\label{eq:somp}
    \{(\vec{\mZ}'_i, \vh'_i)\}_{i=1}^{N+M} &= \varphi_o\left(\{(\vec{\mZ}_i,\vh_i)\}_{i=1}^{N+M}, \gE \right).
\end{align}
Specifically, $\varphi_o$ is unfolded as the following MP and aggregation computations:

\begin{align}
\label{eq:vector-interaction}
\vec{\mZ}_{ij} &= [(\vec{\vp}_{j,1} - \vec{\vp}_{i,1}),\vec{\mZ}_{i},\vec{\mZ}_{j}],\\
\label{eq:scalar-interaction}
\vh_{ij} &=  [\|\vec{\vp}_{j,1} - \vec{\vp}_{i,1}\|_2, \vh_{i}, \vh_{j}] , \\
\label{eq:message-mlp}
\vec{\mM}_{ij}, \vm_{ij} &= \phi_{\vec{\vg}}\left(\vec{\mZ}_{ij}, \vh_{ij} \right), \\
\label{eq:message-aggregation}
\vec{\mM}_i, \vm_i &= \sum\nolimits_{j\in\gN(i)}\vec{\mM}_{ij}, \sum\nolimits_{j\in\gN(i)}\vm_{ij},
\end{align}
\begin{equation}
\aligned
\label{eq:update-mlp}
(\vec{\mZ}'_i, \vh'_i) &= (\vec{\mZ}_i, \vh_i)\\
&+ \psi_{\vec{\vg}}\left([\vec{\mM}_{i}, \vec{\mZ}_i], [\vm_{i}, \vh_i] \right), 
\endaligned
\end{equation}
where [ ] denotes the stack along the last dimension, $\gN(i)=\{j: (i,j)\in\gE \}$ is the neighbors of node $i$, and both $\phi_{\vec{\vg}}$ and $\psi_{\vec{\vg}}$ are subequivariant~\cite{han2022learning,chen2023sgrl}, simplified as follows:
\begin{align}
\label{eq:subeq}
\begin{aligned}
    &f_{\vec{\vg}}(\Vec{\mZ},\vh)=[\Vec{\mZ}, \vec{\vg}]\mM_{\vec{\vg}},\\
    &\text{s.t.}\quad \mM_{\vec{\vg}}=  \sigma([\Vec{\mZ},\vec{\vg}]^{\top}[\Vec{\mZ},\vec{\vg}],\vh),
\end{aligned}
\end{align}
where $\sigma\left(\cdot\right)$ is an Multi-Layer Perceptron (MLP) and $[\Vec{\mZ}, \vec{\vg}]\in\sR^{3\times (m+1)}$ is a stack of $\vec\mZ$ and $\vec\vg$ along the last dimension. The specific form of the function is detailed in ~\cref{sec:subeq}.

We establish a subequivariant edge representation $\vec{\mZ}_{ij}$ and invariant edge features $\vh_{ij}$.
The edge features $\vec{\mZ}_{ij}$ and $\vh_{ij}$ are then input into  $\phi_{\vec{\vg}}$, as defined in ~\cref{eq:message-mlp}, to yield vector and scalar messages $\vec{\mM}_{ij}$ and $\vm_{ij}$, respectively. Message aggregation and state updates are performed as outlined in \cref{eq:message-aggregation} and \cref{eq:update-mlp} using the function $\psi_{\vec{\vg}}$, leading to the updated states \( \vec{\mZ}'_i \) and \( \vh'_i \). This process ensures the generation of outputs maintaining the desired subequivariance or invariance properties.

\paragraph{Local Reference Frame Transform}
Following the local entity-level subequivariant MP, we integrate the resultant entity-level information \( \vh'_i \) and \( \vec{\mZ}'_i \) with body-level information to guide body-level control. Here,
\( \vec{\vu}_i \) represents a local reference frame (LRF) transform vector, derived through the linear transformation \( \vec{\vu}_i = \vec{\mZ}'_i \mW_{\vec{\vu}} \), with \( \mW_{\vec{\vu}} \in \sR^{m \times 1} \).

We normalize and orthogonalize the transform vector \( \vec{\vu}_i \) for each entity:
\begin{equation}
\label{eq:ortho}
\begin{aligned}
    \vec{\ve}_{i1} &= \frac{\vec{\vu}_{i} - \langle\vec{\vu}_{i},\vec{\ve}_{i3}\rangle \vec{\ve}_{i3}}{\| \vec{\vu}_{i} - \langle\vec{\vu}_{i},\vec{\ve}_{i3}\rangle \vec{\ve}_{i3} \|}, \\
    \vec{\ve}_{i2} &= \vec{\ve}_{i1} \times \vec{\ve}_{i3},\\
    \vec{\ve}_{i3} &= [0,0,1]^{\top}.
\end{aligned}
\end{equation}
Here, \( \langle\cdot,\cdot\rangle \) denotes the inner product, and \( \times \) the cross product. We refer to the aforementioned procedure as OP, following which we proceed to construct an entity-wise rotation matrix on transform vectors.
\begin{equation}
\label{eq:ort_matrix}
    \mO_i = \left[ \vec{\ve}_{i1}, \vec{\ve}_{i2}, \vec{\ve}_{i3} \right] = \operatorname{OP}(\vec{\vu}_i), \quad i \in \Omega.
\end{equation}

\begin{theorem}
\label{prop:eqo}
The learned entity-wise rotation matrix, denoted as $ \mO_i = \operatorname{OP}(\vec{\vu}_i)$, are \(\text{SO}_{\vec\vg}(3)\)-equivariant, satisfying any transformation $g \in\text{SO}_{\vec\vg}(3)$, $g \cdot  \mO_i = \operatorname{OP}(g \cdot \vec{\vu}_i)$. 
\end{theorem}
\begin{proof}
See \Cref{sec:proof}.
\end{proof}

At this stage, we establish the LRF for each agent \( i \), with \(\vec{\vc}\) as the origin and $\vec{\ve}_{i1}$ as the orientation of the $x$-axis. 
Notably, with the task assignment, the origin of each agent's LRF shifts from the collective average root position $\vec{\bm{c}}= \frac{1}{N+M} \sum_{i=1}^{N+M}\vec{\bm{p}}_{i,1}$,  to the specific entity's position, $\vec{\bm{c}}=\vec{\bm{p}}_{\mathcal{C}(i),1}$.
This LRF construction enables us to achieve invariant observation inputs:
\begin{equation}
\label{eq:lrf}
\vo'_{i} = \mO_i^{\top}\vo_{i}
= [\{\mO_i^{\top} \vec{\mZ}_{i,k}, \vh_{i,k}\}_{k=1}^{K_i}, \mO_i^{\top} \vec{\mZ}_{j,1}, \vh_{j,1}],
\end{equation}
where $j \in \gN(i)$, and [ ] is the stack along the last dimension, with adjustments for relative positioning in assignment as $\vec{\bm{p}}_{i,k} = \vec{\bm{p}}_{i,k} - \vec{\bm{p}}_{\mathcal{C}(i),1}, \vec{\bm{p}}_{j,1} = \vec{\bm{p}}_{j,1} - \vec{\bm{p}}_{\mathcal{C}(j),1}$, thus achieving decoupled translation and rotation invariance.

Our methodology integrates subequivariant information across entity and body levels via LRF Transform. 
While disrupting reflection symmetry, the necessity to construct an orthogonal rotation matrix significantly enhances the capabilities of subequivariant networks.
This emphasis on rotation symmetry substantially outweighs the reduced focus on reflection symmetry, particularly in diminishing the massive search space for optimal policies.
Empirical validations of this enhancement are detailed in~\cref{sec:ablation}.

\paragraph{Body-level Control}
We are now equipped to output the invariant actor policy \( \pi_\theta \) and invariant critic value-function \( V_{\phi} \) for the training objective in \cref{eq:objective}.
For each agent \( i \), the invariant actor policy \( \pi_{\theta_i} \) leverages the invariant \( \vo'_{i} \) and \( \vh'_i \), defined by
\begin{align}
\label{eq:pi}
    \pi_{\theta_i} = \sigma_{\pi_i}(\vo'_{i}, \vh'_i),
\end{align}
where \( \sigma_{\pi_i} \) is a MLP. Here, \( \pi_{\theta_i} \in \sR^{2 \times (K_i-1)} \) represents the loc and scale parameters of a Tanh-Normal Distribution for the \( (K_i-1) \) actuators of agent \( i \). Consequently, each actuator samples its corresponding torque \( a_{i,k} \in [-1,1] \) from this distribution.

Besides, the invariant critic value-function \( V_{\phi} \) utilizes the entity-level invariant features \( \{\vh'_i\}_{i=1}^{N+M} \), formulated as
\begin{align}
\label{eq:value}
    V_{\phi} = \sigma_{V}(\{\vh'_i\}_{i=1}^{N+M}),
\end{align}
where \( \sigma_{V} \) is a MLP, and \( V_{\phi} \in \sR \).
Formal proof of the $\text{SO}_{\vec\vg}(3)$-invariance of the output action and value are presented in \Cref{sec:proof}.

\section{Multi-entity Benchmark (\textsc{MeBen})}
\label{sec:method}

\begin{figure}[t!]
\centering
\includegraphics[width=\linewidth]{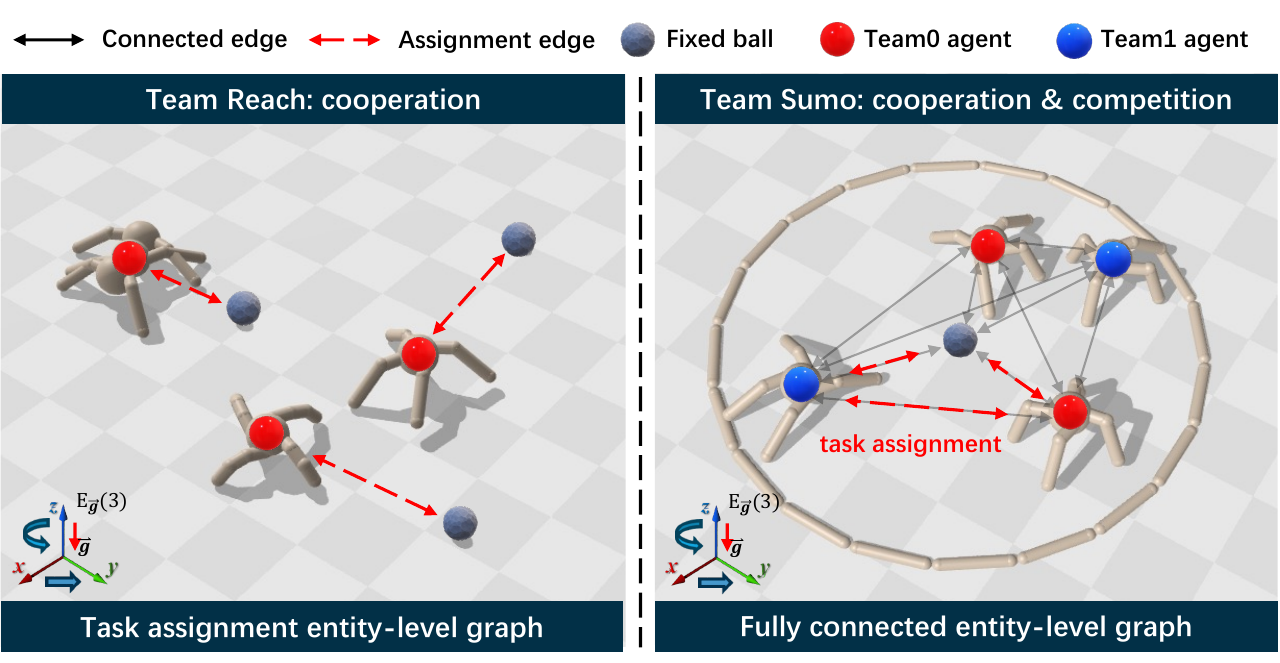}
\vspace{-.25in}
\caption{
Illustration of \textsc{MeBen}: Team Reach (left) where agents cooperate to collectively reach all fixed balls, and Team Sumo (right) where agents engage in both cooperation and competition to push opponents away from the fixed ball. These tasks necessitate the decoupling of local transformations via dynamic task assignment graph construction from the overall structure (depicted by a fully connected graph), while employing $\text{E}_{\vec{\bm{g}}}(3)$-equivariance to effectively compress global state space.
}
\label{fig:meben}
\vspace{-.13in}
\end{figure}

\subsection{Environments descriptions}
\label{sec:bench_construct}
In this subsection, we present the details involved in constructing our environments: Team Reach and Team Sumo, as illustrated in \cref{fig:meben}.

\begin{figure*}[t!]
\centering
\includegraphics[width=\linewidth]{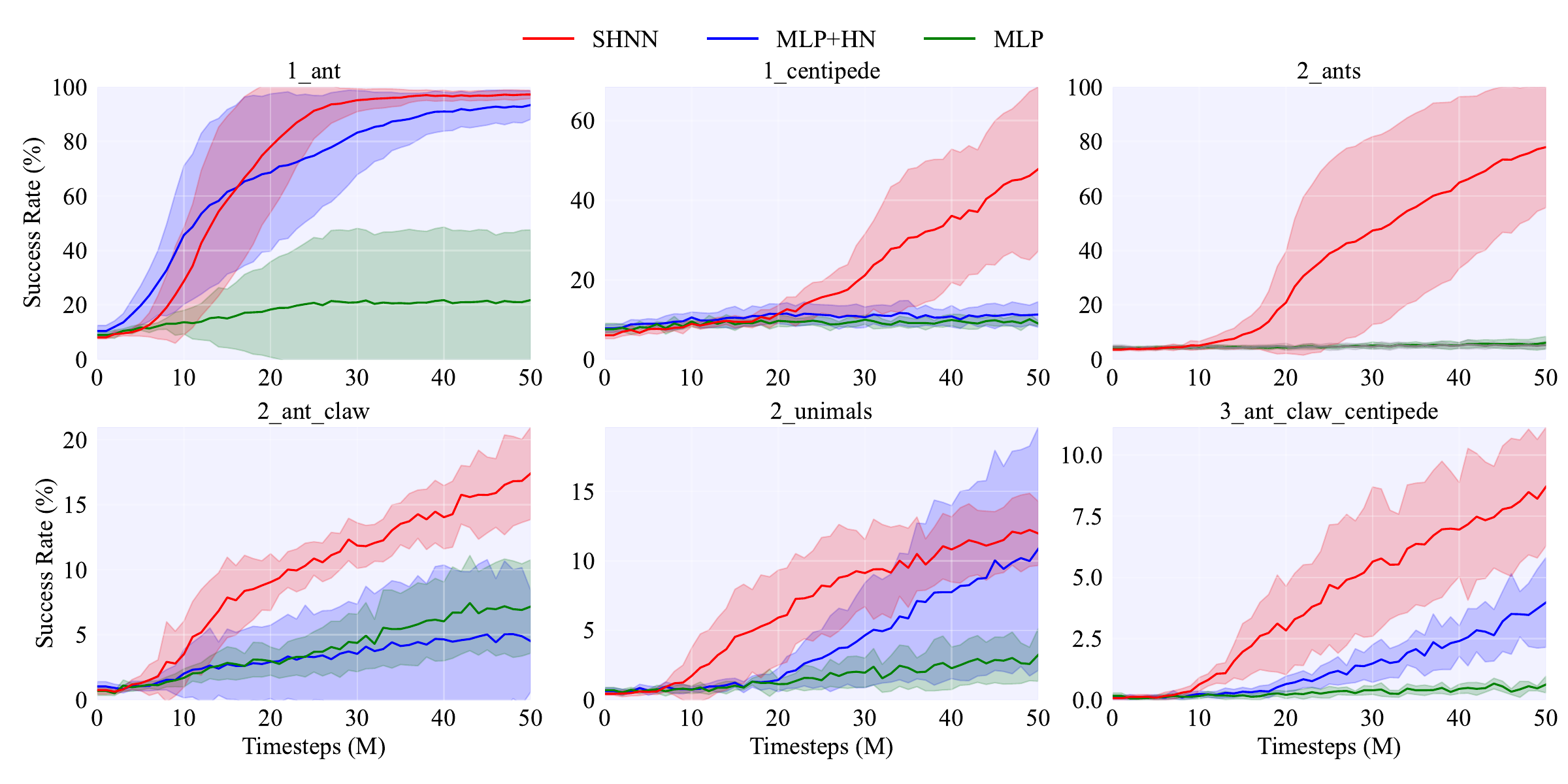}
\vspace{-.3in}
\caption{Training and Evaluation Curves in Team Reach Environments. The shaded area represents the standard error.
}
\label{fig:reach_ben}
\end{figure*}

\paragraph{Agents}
In our studies, we leverage a variety of morphologies, including ants, claws, and centipedes from MxT-Bench~\cite{furuta2023asystem}, as well as unimals from \citet{gupta2022metamorph}. Notably, centipedes and unimals exhibit asymmetrical forms, potentially influencing the overall system's symmetry and dynamics. 
This diverse range of agent morphologies enables a nuanced exploration of agent dynamics in multi-entity, morphology-based RL environments.

\paragraph{Team Reach}
We expand the ``Reach'' task from a single-agent challenge to a collaborative ``Team Reach'' task, as shown in \cref{fig:reach_demo}. 
\textbf{1. Initial Conditions.} Entities set $\Omega$ include \( N \) agents and \( M \) fixed balls, $N \geq M$. Within an area of radius \( R \), we randomly position \( N \) agents and \( M \) fixed balls, also setting the initial orientations of the agents randomly.
Specific details for \( N \), \( M \), \( R \), and the agents' morphology are provided in \cref{tab:environments}.
\textbf{2. Termination.} The goal of this task is for \( M \) fixed balls to be \emph{simultaneously} occupied by agents. An episode is considered successfully completed and terminates once this condition is fulfilled. 
\textbf{3. Reward.}
The designed reward is detailed in \cref{sec:envs}.

\paragraph{Team Sumo}
We evolve the ``Sumo'' task from a purely competitive challenge to a mixed cooperative-competitive ``Team Sumo'' task, as shown in \cref{fig:sumo_demo}.
\textbf{1. Initial Conditions.} Entities set $\Omega$ comprise one fixed ball, \( N \) agents forming Team 1, and another \( M \) agents constituting Team 2.
The sumo arena, a circle with radius \( R \), has its center marked by the fixed ball.
Around this fixed ball, within a radius of \( R-1 \), we randomly position  \( N \) agents from Team 1 and another \( M \) from Team 2, ensuring each agent's orientation is also randomized. 
Specific details for \( N \), \( M \), \( R \), and the agents' morphology are provided in \cref{tab:environments}.
\textbf{2. Termination.} The objective is for either Team 1 or Team 2 to win by having an opposing team's agent disqualified, which occurs if it exceeds a distance \( R \) from the fixed ball. The team with the disqualified agent loses, triggering the termination of the episode. 
\textbf{3. Reward.}
The reward designed for each team is detailed in \cref{sec:envs}.

\begin{table}[t!]
\setlength\tabcolsep{3pt}
    \centering
    \tiny
    \caption{Comparison of Morphology-based Environment Setup.}
    \label{tab:environment_vs}
    \resizebox{\linewidth}{!}{
    \begin{tabular}{lccc}
    \toprule
     \textbf{Aspect} & \textbf{SGRL} & \textbf{MxT-Bench}& \textbf{\textsc{MeBen}} \\
     \midrule
Multi-Morphology     & \checkmark & \checkmark & \checkmark \\
Multi-Agent          & \texttimes & \texttimes & \checkmark \\
Diverse-Task         & \texttimes & \checkmark & \checkmark \\
Supported-Symmetry   & \checkmark & \texttimes & \checkmark \\
Accelerated-Hardware      & \texttimes & \checkmark & \checkmark \\
    \bottomrule
    \end{tabular}
    }
    \label{tab:compare}
\end{table}

\subsection{Design considerations}
\label{sec:MeBen}

\begin{figure*}[t!]
\centering
\includegraphics[width=\linewidth]{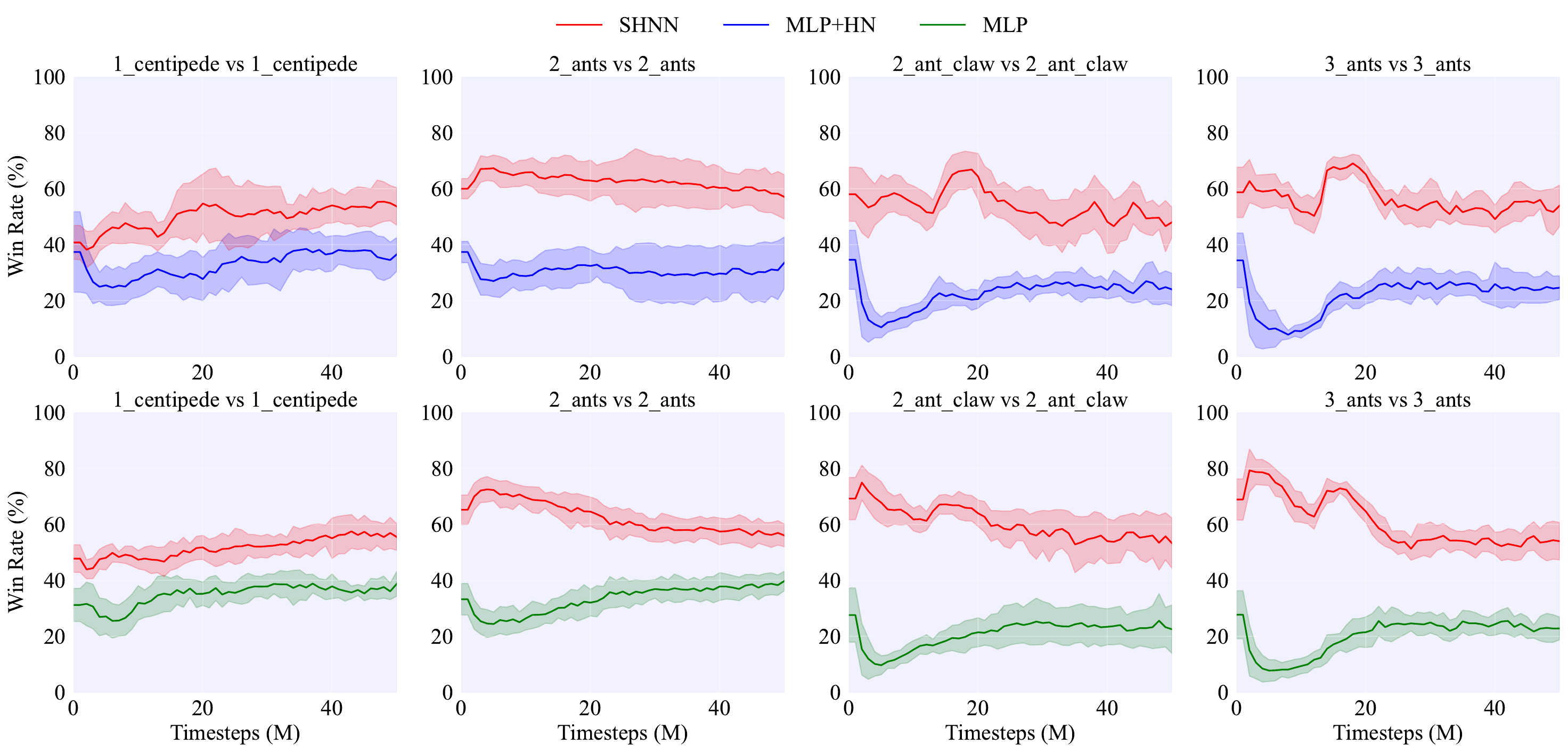}
\vspace{-.32in}
\caption{Training and Evaluation Curves in Team Sumo Environments. The shaded area represents the standard error.
}
\label{fig:sumo_ben}
\end{figure*}

We present the key features and limitations of existing benchmarks in \cref{tab:compare}, comparing them with our introduced \textsc{MeBen}. More comprehensive details are provided below.

\paragraph{Multi-Entity Dynamics} Diverging from the single-agent focus in SGRL~\cite{chen2023sgrl} and MxT-bench~\cite{furuta2023asystem}, \textsc{MeBen} expands to include environments with multiple entities~\cite{spelke2022babies}, enabling a detailed exploration of dynamics and interactions among varied entities, such as agents with complex morphologies, and other significant objects impacting system symmetry and behavior. 
\textsc{MeBen}'s unique design, makes it an ideal platform for a wide range of morphology-based reinforcement learning studies, encompassing single-agent and multi-agent systems (cooperative, competitive, and mixed scenarios). 

\paragraph{Geometric Symmetry}
In contrast to the fixed agent initialization in MxT-bench~\cite{furuta2023asystem}, \textsc{MeBen} provides a more realistic and geometric symmetric setup through its stochastic initial conditions. The positions and orientations of all entities, including agents and objects, are uniformly randomized in \textsc{MeBen}’s environments. This strategy eliminates any bias towards specific orientations or positions, maintaining geometric symmetry and providing a robust test bed for evaluating the generalization capabilities of models in environments with gravity.

\paragraph{Accelerated-Hardware}
In contrast to traditional CPU-based environments such as SGRL~\cite{chen2023sgrl}, \textsc{MeBen} harnesses the advanced capabilities of the Brax physics simulator~\cite{brax2021github} and Composer~\cite{gu2021braxlines}, and further builds upon the MxT-bench framework~\cite{furuta2023asystem}. This advancement propels \textsc{MeBen} into facilitating efficient and scalable hardware-accelerated iterations on GPUs or TPUs, making it exceptionally suitable for morphology-based reinforcement learning experiments.

\begin{table*}[t!]
    \centering
    \scriptsize
    \caption{Evaluations on Basic Architectures. Training and Evaluation in Team Reach Environments. We report Success Rate (\%) on the final step.  }
    \label{tab:transformer}
    \begin{tabular}{lcccccc}
    \toprule
    \textbf{Methods} & \texttt{1\_ant} & \texttt{1\_centipede} & \texttt{2\_ants}& \texttt{2\_ant\_claw}& \texttt{2\_unimals}& \texttt{2\_ant\_claw\_centipede}\\ \midrule
MLP+HN               & 93.39 $\pm$ 5.25       & 11.28 $\pm$ 3.21        & 5.25 $\pm$ 1.399       & 4.52 $\pm$ 3.93         & 10.86 $\pm$ 8.74       & 3.98 $\pm$ 1.83                     \\
SHNN                 & \textbf{97.26} $\pm$ 1.51       & \textbf{47.82} $\pm$ 20.62       & \textbf{77.93} $\pm$ 22.22      & \textbf{17.40} $\pm$ 3.54        & \textbf{11.97} $\pm$ 2.31       & \textbf{8.70} $\pm$ 2.42                     \\   
\midrule
Transformer+HN       & 5.47 $\pm$ 2.84        & 5.55 $\pm$ 2.99         & 2.47 $\pm$ 0.73        & 0.51 $\pm$ 0.19         & 0.25  $\pm$ 0.10        & 2.26 $\pm$ 1.92                     \\
SHTransformer        & \textbf{63.61} $\pm$ 39.57      & \textbf{11.52} $\pm$ 2.39        & \textbf{11.37} $\pm$ 15.50      & \textbf{1.17} $\pm$ 0.70         &\textbf{ 0.26} $\pm$ 0.15        & \textbf{7.12} $\pm$ 2.29                     \\ 
    \bottomrule
    \end{tabular}
\end{table*}

\section{Experiments}

\subsection{Experimental  Setup}

\paragraph{Baselines}  We compare our method SHNN, against mainstream neural networks, particularly MLP ~\cite{furuta2023asystem}, and its variant utilizing heading normalization tricks, denoted as MLP+HN~\cite{chen2023sgrl}. 
Please refer to \cref{sec:baselines} for details about baselines. 

\paragraph{Metrics}   \textbf{1. Team Reach: } \textit{Success Rate} = \#\textit{success\_episode} / \#\textit{evaluation\_episode}. 
 \textbf{2. Team Sumo:} 
For each team, \textit{Win Rate} = \#\textit{win\_episode} / \#\textit{evaluation\_episode}.
Each experiment is conducted with 10 seeds to report the \textit{Success (or Win) Rate} over \#\textit{evaluation\_episode} of $1024$.

\paragraph{Implementations} 
The environments used in this work are detailed in \cref{tab:environments}. To ensure fairness, all baselines, ablations, and our SHNN model use the same input information and employ MAPPO~\cite{yu2022surprising} as the training algorithm for MARL. SHNN is developed based on the MxT-bench\cite{furuta2023asystem} codebase, leveraging JAX\cite{jax2018github} and Brax~\cite{brax2021github,gu2021braxlines} for efficient, hardware-accelerated simulations. 
The value of the maximum timesteps per episode is $1000$.
Hyperparameter details are in \cref{sec:hypers}. Within the same team, agents share only the weights of entity-level message passing, not the body-level MLP. 
For the Team Reach environments, we construct a directed graph by dynamically assigning a fixed ball to each agent.
For the Team Sumo environments, we assign each agent an opponent from the opposing team and assign each with the arena's center ball.
Additionally, in Team Sumo environments, we adopt \citet{bansal2018emergent}'s approach where teams employing baseline methods and SHNN compete within an arena.

\subsection{Evaluations in Diverse Environments}
We begin by evaluating our method on diverse multi-entity tasks in 3D environments.
\textbf{Team Reach} in \Cref{fig:reach_ben}:
\textbf{1.} The MLP generally fails to achieve meaningful returns in most Team Reach cases, which is due to its vulnerability to local extremes within the expansive exploration space of our environments.
\textbf{2.} While MLP enhanced with Heading Normalization (MLP+HN) shows close performance to SHNN in specific environments like \texttt{1\_ant} and \texttt{2\_unimals}, their effectiveness diminishes in other scenarios. This variation in performance can be attributed to factors such as the symmetry in morphology and the complexities arising from multi-agent interactions.
\textbf{3.} Our proposed SHNN demonstrates a clear advantage, outperforming all baselines across various scenarios. 
\textbf{Team Sumo} in \Cref{fig:sumo_ben}:
with an increase in the number and complexity of agent morphologies, the advantage of our method over baselines progressively amplifies. 
These observations underscore that our proposed SHNN method not only exhibits superior capabilities in compressing the global state space of multi-entity tasks and significantly enhances the generalization of body-level policies, but also sees these advantages magnified in scenarios with increasingly complex agent morphologies and a greater number of agents.

\subsection{Extended Evaluations on Transformer}
\label{sec:transformer}

Additionally, building upon our initial exploration with MLP, our exploration extended to Transformers, utilizing the same backbone as Amorpheus~\cite{kurin2020cage} and MxT-Bench~\cite{furuta2023asystem} for $\sigma_{\pi_i}$ in \cref{eq:pi} and $\sigma_{V}$ in \cref{eq:value}.
Despite Transformers' advances in morphology-based RL~\cite{kurin2020cage,hong2021structure,dong2022low,chen2023sgrl,furuta2023asystem}, their training complexity and computational demands limit their performance improvement over MLPs in environments with restricted agent morphology diversity, as detailed in \cref{tab:model_comparison} and \cref{tab:transformer}, with curves in \Cref{fig:transformer}. 
Moreover, SHTransformer (our plug-in applied to body-level Transformer control) consistently outperforms standard Transformer models. This evidence confirms the plug-in's broad applicability and effectiveness across varied architectural frameworks.

\begin{figure}[t!]
\centering
\includegraphics[width=\linewidth]{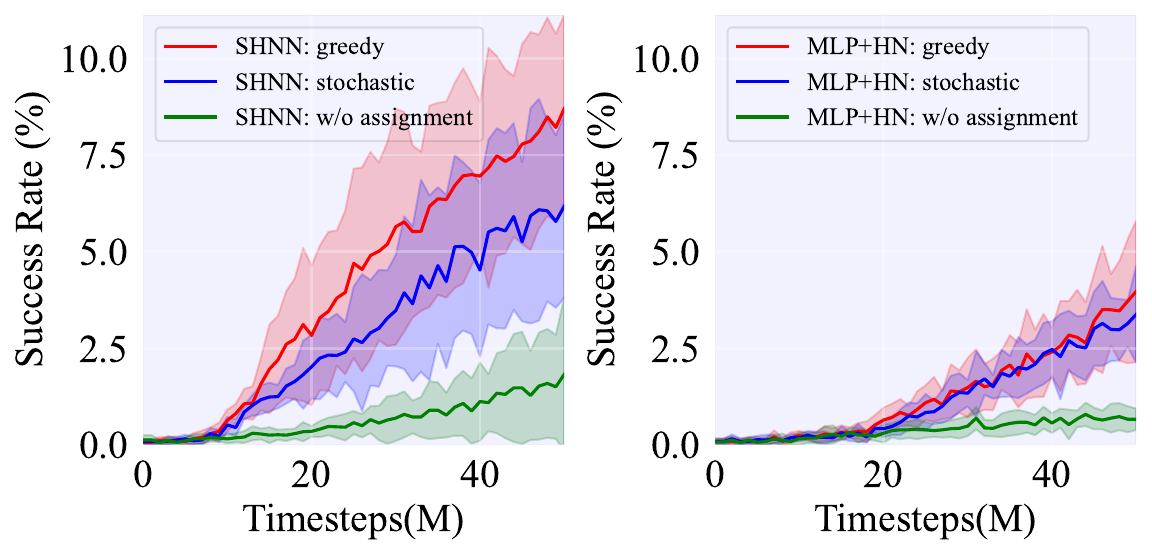}
\vspace{-.2in}
\caption{Ablations on Assignment. Training and Evaluation Curves in \texttt{3\_ant\_claw\_centipede} Team Reach. 
}
\label{fig:part_seg}
\end{figure}

\begin{figure}[t!]
\centering
\includegraphics[width=\linewidth]{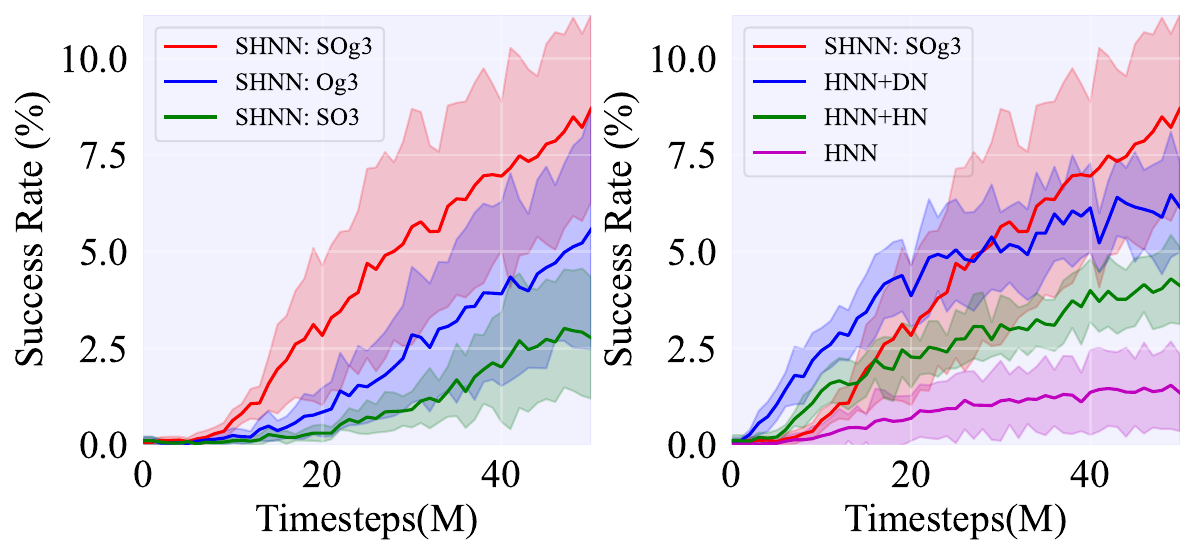}
\vspace{-.2in}
\caption{Ablations on Equivariance. Training and Evaluation Curves in \texttt{3\_ant\_claw\_centipede} Team Reach. 
}
\label{fig:reach_ablation}
\end{figure}

\subsection{Ablation Studies}

\paragraph{Ablations on Assignment}
\label{sec:oas}
We evaluate the impact of task assignment by comparing several variants, as depicted in \cref{fig:part_seg} and \cref{fig:sumo_ben_bipair}. We consider the following configurations:
\textbf{greedy}, a default model that employs a greedy strategy-based bipartite matching for task assignment;
\textbf{stochastic}, a variant based on stochastic assignments;
\textbf{w/o assignment}, a variant where no task assignment is performed, with the entity-level graph configured as a fully connected graph and sparse root mean $\vec{\bm{c}}= \frac{1}{N+M} \sum_{i=1}^{N+M}\vec{\bm{p}}_{i,1}$ serving as the origin of LRF in ~\cref{eq:lrf}.
Our empirical findings affirm that assignment consistently enhances performance across various neural networks architectures in the Team Reach environments. Specifically, MLP networks exhibit marked improvement when incorporating task assignments. For SHNN, the greedy graph matching strategy of task assignment significantly outperforms both stochastic and non-assignment strategies. This evidence underscores that our task assignment effectively facilitates task-specific dynamics, managing complex entity interactions by decoupling local transformations from the overall structures while compressing the search space through local physical symmetry.
For Team Sumo, see \cref{sec:add_cluster}.

\paragraph{Ablations on Equivariance}
\label{sec:ablation}
We ablate the following variants in \Cref{fig:reach_ablation} and \Cref{fig:sumo_ablation}: 
\textbf{SHNN:SOg3}, our full model, which is $\text{SO}_{\vec{\vg}}(3)$-equivariant;
\textbf{SHNN:Og3}, an $\text{O}_{\vec{\vg}}(3)$-equivariant variant, where $\mO_i = \vec{\mZ}'_i$ in \cref{eq:ort_matrix};
\textbf{SHNN:SO3}, a $\text{SO}(3)$-equivariant variant, where $\vec\vg$ is omitted in the computation in \cref{eq:subeq};
\textbf{HNN+DN}, a non-equivariant variant, where $\vec{\mZ}$ are treated as scalars, but which utilizes the goal direction to construct the LRF;
\textbf{HNN+HN}, a non-equivariant variant, where $\vec{\mZ}$ are treated as scalars, but which utilizes the agent's heading direction to construct the LRF;
\textbf{HNN}, a non-equivariant variant, where $\vec{\mZ}$ are treated as scalars.    
\textbf{1.} \emph{Which symmetry group works best within our network framework?} Comparative experiments between the $\text{SO}_{\vec{\vg}}(3)$ group and the $\text{O}_{\vec{\vg}}(3)$ group show that the positive impact of emphasizing rotational symmetry significantly outweighs any disadvantages from the reduced focus on reflection symmetry, particularly in terms of reducing the massive search space. Furthermore, the comparison between the $\text{SO}_{\vec{\vg}}(3)$ group and the $\text{SO}(3)$ group demonstrates the importance of sensing the direction of gravity in policy learning.
\textbf{2.} \emph{Can equivariant network methods replace or even surpass methods based on hand-crafted LRF?} Experiments demonstrate that whether using goal orientation or the agent's heading orientation to construct the LRF, the performance is inferior compared to ours. This indicates that for current mainstream neural networks, entity-level subequivariant message passing emerges as a simpler yet more effective alternative to hand-crafted LRF, providing a plug-in solution for equivariant modifications.

\begin{figure}[t!]
\centering
\includegraphics[width=0.8\linewidth]{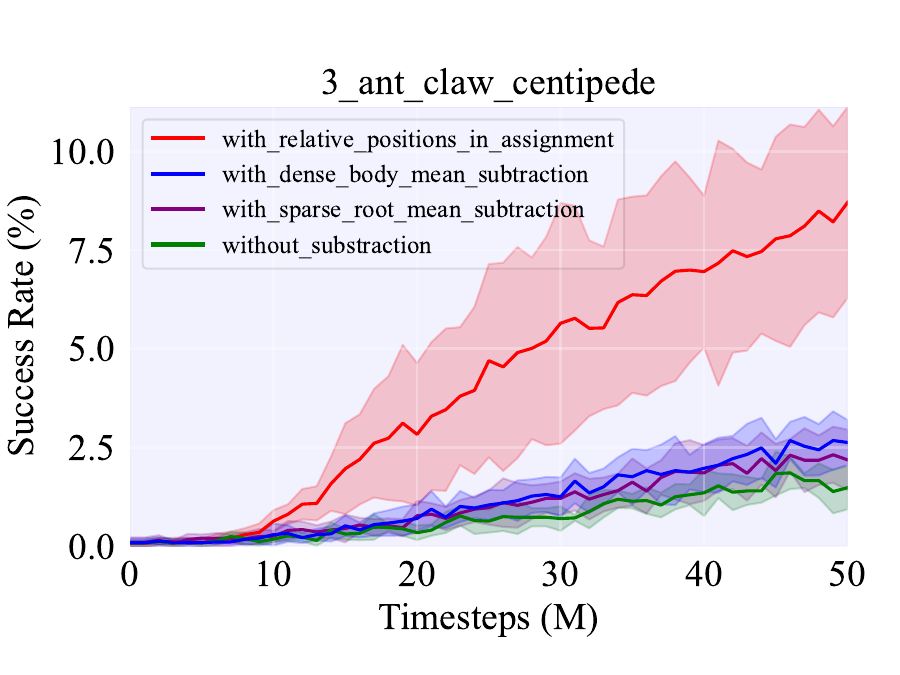}
\vspace{-.2in}
\caption{Local Translation Invariance. Training and Evaluation Curves in \texttt{3\_ant\_claw\_centipede} Team Reach Environment.
}
\label{fig:ablation_translation_3cc}
\vspace{-.13in}
\end{figure}

\paragraph{Importance of Local Symmetry}
The analysis of local symmetry in the Team Reach environments, illustrated by comparing the red and green lines in the left plot of \cref{fig:part_seg}, and the red and purple lines in the right plot of \Cref{fig:reach_ablation}, underscores the indispensable roles of task assignment and equivariance. Their integration is pivotal for surmounting the benchmark's challenges, validating the SHNN method's design philosophy.
To further substantiate the criticality of local symmetry, we embarked on ablation studies exploring the influence of local translation invariance with the following variants: \textbf{without subtraction}, where no translation invariance implemented; 
\textbf{with sparse root mean subtraction}:  $\vec{\bm{c}}= \frac{1}{N+M} \sum_{i=1}^{N+M}\vec{\bm{p}}_{i,1}$, i.e.,  without decoupling translation transformations;
\textbf{with dense body mean subtraction}, $\vec{\bm{c}}= \frac{1}{\sum_{i=1}^{N+M}\sum_{k=1}^{K_i}1} \sum_{i=1}^{N+M}\sum_{k=1}^{K_i}\vec{\bm{p}}_{i,k}$, i.e.,  without decoupling translation transformations; 
\textbf{with relative positions in assignment}, $\vec{\bm{c}}=\vec{\bm{p}}_{\mathcal{C}(i),1}$, i.e., with decoupling of translation transformations.
The results in \cref{fig:ablation_translation_3cc} clearly indicate that without achieving translation invariance or properly decoupling it, the model struggles to mitigate complexity challenges, leading to significantly poorer performance. This observation further emphasizes the pivotal importance of local symmetry.

\subsection{Analyses of Morphology-shared Policy}

\begin{table}[t!]
    \centering
    \scriptsize
    \caption{Analyses of Morphology-shared Policy. We report Success Rate (\%) on the final step in Team Reach Environments. }
    \label{tab:low_graph}
    \begin{tabular}{lcc}
    \toprule
    \textbf{Methods} & \texttt{2\_ants} & \texttt{3\_ant\_claw\_centipede} \\ \midrule
    SHNN & $\textbf{77.93} \pm 22.22$ & $8.70 \pm 2.42$ \\ 
    SHGNN& $41.41 \pm 29.47$ & $\textbf{31.07} \pm 26.84$ \\
    \bottomrule
    \end{tabular}
    \vspace{-.13in}
\end{table}

Previous works have achieved generalization across agents with different morphologies using morphology-aware Graph Neural Networks \cite{wang2018nervenet,huang2020one}.
In addition to the previously mentioned SHTransformer (our plug-in applied to body-level Transformer control), we have also developed a new variant, denoted as SHGNN (our plug-in applied to body-level message passing control). This variant learns a shared policy network across different agents, detailed in \cref{sec:sgnn}.
As observed in \cref{tab:low_graph} and \cref{tab:model_comparison}, despite body-level message passing excelling in complex, multi-morphological environments by enhancing knowledge sharing and performance, its implementation significantly slows down the training process.
Above all, our benchmark provides the community with a test bed that takes into account both agent-interaction and morphology-aware considerations.
\section{Limitations and Future Works}
To learn policies in 3D multi-entity physical environments, we propose the SHNN, a framework that uniquely integrates task assignment with local subequivariant message passing in a hierarchical structure. 
However, in the Team Sumo environments, as shown in \cref{fig:sumo_ben_bipair}, the impact of assignment is not as significant as that of equivariance. 
Therefore, the interdependence of task assignment and equivariance necessitates a novel co-learnable formulation.
Our model's reliance on the $\text{E}_{\vec{\bm{g}}}(3)$ symmetry encompasses a broad range of physical interactions but does not extend to non-Euclidean or highly irregular environments not governed by classical physics. Additionally, our model predominantly leverages state-based inputs, which are often costly to acquire in practice due to the need for precision sensors. Addressing the challenge of applying our equivariance-focused approach to vision-based inputs remains an open area for future research.

\section*{Acknowledgements}
This work is funded by 
the National Science and Technology Major Project of the Ministry of Science and Technology of China (No.2018AAA0102903),
and  partly by
THU-Bosch JCML Center.
We sincerely thank the reviewers for their comments that significantly improved our paper's quality. 

\section*{Impact Statement}
The research presented in this paper is a step toward advancing the field machine learning~\cite{yin2024distributed,guan2024automatic}. While our contributions focus on the theoretical and technical advancements in machine learning, we acknowledge the broader societal implications as our work indirectly supports the development of technologies in delivery, transportation, and automated systems. The practical applications, such as in autonomous vehicles and drones, offer promising avenues for societal benefits, including improved safety and efficiency. While our work heralds technical progress, we are mindful of its societal impacts, notably privacy concerns from using detailed state data, emphasizing the need for robust policies to ensure these technologies' ethical and secure societal integration.


\bibliography{example_paper}
\bibliographystyle{icml2024}

\newpage
\appendix
\onecolumn

\section{Proofs}
\label{sec:proof}
In this section, we theoretically prove that our proposed SHNN ensures the final output action and critic value preserve the symmetry as desired. 

\begin{theorem}
\label{prop:eqo1}
The learned entity-wise rotation matrix, denoted as $ \mO_i = \operatorname{OP}(\vec{\vu}_{i})$, are \(\text{SO}_{\vec{\vg}}(3)\)-equivariant, satisfying any transformation $g \in\text{SO}_{\vec{\vg}}(3)$, $g \cdot  \mO_i = \operatorname{OP}(g \cdot \vec{\vu}_{i})$. 
\end{theorem}
\begin{proof}
To prove that the entity-wise rotation matrix $\mO_i = \operatorname{OP}(\vec{\vu}_{i})$ are \(\text{SO}_{\vec{\vg}}(3)\)-equivariant, we need to show that under any transformation $g \in \text{SO}_{\vec{\vg}}(3)$, the transformation of $\mO_i$ through $g$ is equivariant to the rotation matrix obtained from the transformed vectors $g \cdot \vec{\vu}_{i}$.

Let $g$ be a transformation in $\text{SO}_{\vec{\vg}}(3)$, which includes a rotation $\mO$ along the direction of $\vec\vg$. 
Specifically, the transformation is applied as follows:
\begin{align*}
    g \cdot \mO_i &= \mO \mO_i, \\
    g \cdot \vec{\vu}_{i} &= \mO \vec{\vu}_{i}, 
\end{align*}

Let $\mO_i^\ast = \operatorname{OP}(g \cdot  \vec{\vu}_{i})$. By the properties of the  orthogonalization process, we have:
\begin{align*}
    \vec{\ve}_{i1}^\ast &= \frac{\mO\vec{\vu}_{i} - \langle\mO\vec{\vu}_{i},\vec{\ve}_{i3}^\ast\rangle \vec{\ve}_{i3}^\ast}{\| \mO\vec{\vu}_{i} - \langle\mO\vec{\vu}_{i},\vec{\ve}_{i3}^\ast\rangle \vec{\ve}_{i3}^\ast \|}, \\
    \vec{\ve}_{i2}^\ast &= \vec{\ve}_{i1}^\ast \times \vec{\ve}_{i3}^\ast,\\
    \vec{\ve}_{i3}^\ast &= [0,0,1]^{\top}.
\end{align*}

Since $\mO\vec\vg=\vec\vg$, $\mO^\top\mO=\mI$ and $\det(\mO) = 1$, it preserves inner products and norms. And, the cross product obeys the following identity under matrix transformations:
$(M \Vec{\va}) \times(M \Vec{\vb})=(\operatorname{det} M)\left(M^{-1}\right)^{\mathrm{T}}(\Vec{\va} \times \Vec{\vb})$.
Therefore, $\vec{\ve}_{i3}^\ast=\mO\vec{\ve}_{i3}$, $\|\mO \vec{\vu}_{i}\| = \|\vec{\vu}_{i}\|$ and $\langle\mO\vec{\vu}_{i},\mO\vec{\ve}_{i3}\rangle = \langle\vec{\vu}_{i},\vec{\ve}_{i3}\rangle$. This implies:

\begin{align*}
    \vec{\ve}_{i1}^\ast &= \frac{\mO\vec{\vu}_{i} - \langle\mO\vec{\vu}_{i},\mO\vec{\ve}_{i3}\rangle \mO\vec{\ve}_{i3}}{\| \mO\vec{\vu}_{i} - \langle\mO\vec{\vu}_{i},\mO\vec{\ve}_{i3}\rangle \mO\vec{\ve}_{i3} \|} = \frac{\mO\vec{\vu}_{i} - \langle\vec{\vu}_{i},\vec{\ve}_{i3}\rangle \mO\vec{\ve}_{i3}}{\| \mO\vec{\vu}_{i} - \langle\vec{\vu}_{i},\vec{\ve}_{i3}\rangle \mO\vec{\ve}_{i3} \|} = \frac{\mO (\vec{\vu}_{i} - \langle\vec{\vu}_{i},\vec{\ve}_{i3}\rangle \vec{\ve}_{i3})}{\| \vec{\vu}_{i} - \langle\vec{\vu}_{i},\vec{\ve}_{i3}\rangle \vec{\ve}_{i3} \|} = \mO \vec{\ve}_{i1}, \\
    \vec{\ve}_{i2}^\ast &= \mO\vec{\ve}_{i1} \times \mO\vec{\ve}_{i3}= \det(\mO)\mO (\vec{\ve}_{i1} \times \vec{\ve}_{i3}) = \mO \vec{\ve}_{i3},\\
    \vec{\ve}_{i3}^\ast & = [0,0,1]^{\top} = \mO \vec{\ve}_{i3}.
\end{align*}

Therefore, with $\mO_i^\ast  = \operatorname{OP}(g \cdot  \vec{\vu}_{i}) =  \left[ \mO\vec{\ve}_{i1}, \mO\vec{\ve}_{i2}, \mO\vec{\ve}_{i3} \right] = \mO \mO_i = g \cdot \mO_i$, we confirm the \(\text{SO}_{\vec{\vg}}(3)\)-equivariance of $\mO_i$. 

Besides, for reflection transformations $\mO$, characterized by $\mO^\top\mO=\mI$ and $\det(\mO) = -1$, $\vec{\ve}_{i2}^\ast = -\mO \vec{\ve}_{i2} \rightarrow \mO_i^\ast \neq \mO \mO_i$ where the equivariance is disrupted due to the properties of the cross product. Hence, although we utilize an $\text{O}_{\vec{\vg}}(3)$-equivariant message passing network in \cref{eq:subeq}, the orthogonalization process can transform a $\text{O}_{\vec{\vg}}(3)$-equivariant vector into a $\text{SO}_{\vec{\vg}}(3)$-equivariant matrix.
\end{proof}

As for the actor and critic, we additionally have the following corollary.
\begin{corollary}
\label{prop:act-equ}
Let $\pi_{\theta_i}, V_{\phi}$ be output of the actor and the critic of SHNN with $\vec\mZ, \vec\vg, \vh$ as input. Let $\pi_{\theta_i}^\ast,  V_{\phi}^\ast$ be the actor and critic with $\mO\vec\mZ, \mO \vec\vg, \vh$ as input, $\mO\in\text{SO}_{\vec\vg}(3)$. Then, $(\pi_{\theta_i}^\ast, V_{\phi}^\ast)=(\pi_{\theta_i}, V_{\phi})$, indicating the output actor and critic preserve $\text{SO}_{\vec\vg}(3)$-invariance.
\end{corollary}

\begin{proof}
Given the subequivariance of \cref{eq:subeq} as per ~\cite{han2022learning,chen2023sgrl}, we have ${\vh'_i}^\ast = \vh'_i$.

Hence,
\begin{align}
    V_{\phi}^\ast &= \sigma_{V}(\{{\vh'_i}^\ast\}_{i=1}^{N+M})\\
    &= \sigma_{V}(\{\vh'_i\}_{i=1}^{N+M})\\
    &= V_{\phi}.
\end{align}

Then, we get $\vo_{i}^\ast$ with $\mO\vec\mZ, \mO \vec\vg, \vh$ as input:
\begin{align}
    \vo_{i}^\ast &= [\{\mO \vec{\mZ}_{i,k}, \vh_{i,k}\}_{k=1}^{K_i}, \mO \vec{\mZ}_{j,1}, \vh_{j,1}].
\end{align}

By \Cref{prop:eqo1}, we can obtain $\mO_i^\ast = \mO \mO_i$ with $\mO\vec\mZ, \mO \vec\vg, \vh$ as input. 

Therefore, the LRF invariant observation inputs ${\vo'_{i}}^\ast$  with $\mO\vec\mZ, \mO \vec\vg, \vh$ as input:
\begin{align}
    {\vo'_{i}}^\ast &= {\mO_i^\ast}^{\top}\vo_{i}^\ast\\
    &= [\{{\mO_i^\ast}^{\top} \mO\vec{\mZ}_{i,k}, \vh_{i,k}\}_{k=1}^{K_i}, {\mO_i^\ast}^{\top} \mO\vec{\mZ}_{j,1}, \vh_{j,1}]\\
    &= [\{\mO_i^{\top} \mO^{\top} \mO \vec{\mZ}_{i,k}, \vh_{i,k}\}_{k=1}^{K_i}, \mO_i^{\top} \mO^{\top} \mO \vec{\mZ}_{j,1}, \vh_{j,1}]\\
    &= [\{\mO_i^{\top}  \vec{\mZ}_{i,k}, \vh_{i,k}\}_{k=1}^{K_i}, \mO_i^{\top}  \vec{\mZ}_{j,1}, \vh_{j,1}]\\
    &= \mO_i^{\top}\vo_{i} = \vo'_{i}.
\end{align}

Hence,
\begin{align}
    \pi_{\theta_i}^\ast &= \sigma_{\pi_i}({\vo'_{i}}^\ast, {\vh'_i}^\ast)\\
    &= \sigma_{\pi_i}(\vo'_{i}, \vh'_i)\\
    &= \pi_{\theta_i}.
\end{align}
\end{proof}

\section{Related Works}
\label{sec:related}

\paragraph{Morphology-based RL}
In the field of reinforcement learning (RL), recent years have witnessed the emergence and evolution of morphology-based approaches, particularly within the context of inhomogeneous morphology settings. This setting is distinguished by varying state and action spaces across different tasks~\cite{devin2017learning,chen2018hardware,d2020sharing}. Morphology-based RL decentralizes the control of multi-joint robots by learning a shared policy for each joint, applying a multitude of message-passing strategies. To tackle the challenges of the inhomogeneous setting, methods such as NerveNet~\cite{wang2018nervenet}, DGN~\cite{pathak2019learning}, and SMP~\cite{huang2020one} represent the agent's morphology as a graph and implement Graph Neural Networks (GNNs) for their policy networks. On the other hand, Amorpheus~\cite{kurin2020cage}, SWAT~\cite{hong2021structure}, ModuMorph~\cite{xiong2023universal}, SGRL~\cite{chen2023sgrl}, and Solar~\cite{dong2022low} opt for transformers over GNNs for direct communication. Both approaches demonstrate that graph-based policies offer significant benefits compared to conventional monolithic policies. Furthermore, MxT-Bench~\cite{furuta2023asystem} serves as a testbed for morphology-task generalization, though its primary focus remains on single-agent tasks. Our work, however, expands upon these methodologies to encompass multi-entity environments, facilitating comprehensive exploration of dynamics and interactions among a variety of entities, including agents with intricate morphologies, objects, and other crucial factors affecting system symmetry and behaviors.

\paragraph{Multi-Agent RL}
In the realm of Multi-Agent RL (MARL), decision-making is predominantly guided by frameworks such as the decentralized partially observable Markov decision process (Dec-POMDP) \cite{bernstein2002complexity, oliehoek2016concise,lechner2023gigastep}, with strategies typically categorized into cooperative \cite{xu2023haven, rashid2018qmix}, competitive \cite{bansal2018emergent}, and mixed interactions \cite{lowe2017multi}. Addressing challenges like high-dimensional action spaces and the necessity for agent coordination, decentralized learning \cite{littman1994markov, foerster2016learning} aims to independently optimize each agent's policy. While scalable, this approach often struggles with issues of instability, especially under conditions of partial observability. On the other hand, centralized learning \cite{claus1998dynamics, kraemer2016multi} offers a more comprehensive view but faces computational hurdles in complex environments. The Centralized Training with Decentralized Execution (CTDE) paradigm \cite{lowe2017multi,sunehag2018value,rashid2018qmix,foerster2018counterfactual,iqbal2019actor,kuba2021trust,wen2022multi,yu2022surprising,jeon2022maser} 
emerges as a balanced solution to these extremes. It allows agents to make decisions based on local observations, while global state information is utilized in the construction of the value function. This synergy of individual autonomy and collective insight enhances the efficacy of the decision-making process. Notably, \citet{yu2022surprising} revisited the application of Proximal Policy Optimization (PPO) in MARL within the CTDE structure, achieving remarkably strong performance. 
Moreover, existing multi-agent reinforcement learning (MARL) environments~\cite{samvelyan2019starcraft,de2020deep,ellis2022smacv2,flair2023jaxmarl,lechner2023gigastep} often lack scenarios that encompass both geometric symmetry and morphology-based control.
We expand existing environments~\cite{chen2023sgrl,furuta2023asystem} into a suite of new multi-entity benchmark (MeBen) in 3D space, not only facilitating a comprehensive exploration of multi-entity dynamics but also ensuring a realistic emulation of real-world scenarios through randomized initial conditions and orientations. 
Our work aligns with the Dec-POMDP framework and capitalizes on the strengths of MAPPO \cite{yu2022surprising} to optimize policies effectively. Furthermore, our environment is implemented using JAX~\cite{jax2018github}, ensuring efficient simulations on advanced hardware accelerators such as GPUs and TPUs.

\paragraph{Geometrically Equivariant Models}
The physical world exhibits specific symmetries, extensively explored in studies on group equivariant models~\cite{cohen2016group, cohen2016steerable, worrall2017harmonic}. Building upon this, the field of geometrically equivariant graph neural networks~\cite{han2024survey} has emerged, utilizing symmetry as a fundamental bias in learning. These models are designed to ensure that outputs will rotate, translate, or reflect in the same manner as their inputs, thereby retaining inherent symmetries. Techniques such as group convolution via irreducible representation~\cite{thomas2018tensor,fuchs2020se3} and invariant scalarization methods, like inner product computation~\cite{villar2021scalars,satorras2021en,huang2022equivariant,han2022learning}, are employed to achieve this symmetry preservation.
Our method, similar to GMN~\cite{huang2022equivariant} and SGNN~\cite{han2022learning}, especially focuses on scalarization strategies. In Markov decision processes (MDPs) that exhibit symmetries~\cite{van2020mdp}, these symmetries in the state-action space enable the optimization of policies within a simplified abstract MDP. The work of \citet{van2020mdp} concentrates on learning equivariant policies and invariant value networks in 2D environments. In contrast, \citet{chen2023sgrl} explores body-level equivariant policy networks in more complex 3D physics environments, facilitating policy generalization across different directions.
Our work diverges by introducing an entity-level subequivariant message-passing mechanism, which proves to be highly effective in 3D multi-body scenes~\cite{han2022learning}.

\section{More Experimental Details}\label{sec:details}

\subsection{Subequivariant Function}\label{sec:subeq}
We resort to subequivariant function with $\vec\mZ, \vec\vg, \vh$ as input~\cite{han2022learning,chen2023sgrl} to instill desired geometric symmetry into the model:
\begin{align}
\label{eq:subeq2}
\begin{aligned}
    &\vec{\mZ}', \vh'=\Vec{\mM}_{\vec{\vg}}\mW_{\vec{\vg}}, \mW_{\vec{\vg}},\\
    &\text{s.t.}\quad \mW_{\vec{\vg}}=  \sigma(\Vec{\mM}_{\vec{\vg}}^{\top}\Vec{\mM}_{\vec{\vg}},\vh),
\end{aligned}
\end{align}
where $\Vec{\mM}_{\vec{\vg}}=[\Vec{\mZ}, \vec{\vg}]\mW$is a mixing of the vectors to capture the interactions between channels, with a learnable weight matrix $\mW \in\sR^{(m+1)\times m}$ and $[\Vec{\mZ}, \vec{\vg}]\in\sR^{3\times (m+1)}$ is a stack of $\vec\mZ$ and $\vec\vg$ along the last dimension. The inner product $\Vec{\mM}_{\vec{\vg}}^{\top}\Vec{\mM}_{\vec{\vg}}\in\sR^{m\times m}$ is computed and concatenated with $\vh$. The resultant invariant term is then transformed by a Multi-Layer Perceptron (MLP)) $\sigma:\sR^{m\times m + h}\mapsto\sR^{m\times m}$ producing $\mW_{\vec{\vg}}\in\R^{m\times m}$. 

\subsection{Environments}\label{sec:envs}
In this subsection, we present the technical details involved in constructing our challenging \textsc{MeBen}. The environments utilized in this work are listed in \cref{tab:environments}.

\paragraph{Agents}
In our studies, we leverage a variety of morphologies, including ants, claws, and centipedes from MxT-Bench~\cite{furuta2023asystem}, as well as unimals from \citet{gupta2022metamorph}. Significantly, the asymmetrical forms of centipedes and unimals have the potential to influence the overall system’s symmetry and dynamics.  
This diversity in agent morphologies facilitates a nuanced exploration of agent dynamics within multi-entity, morphology-based RL environments.

\begin{figure}[ht]
\centering
\includegraphics[width=0.8\linewidth]{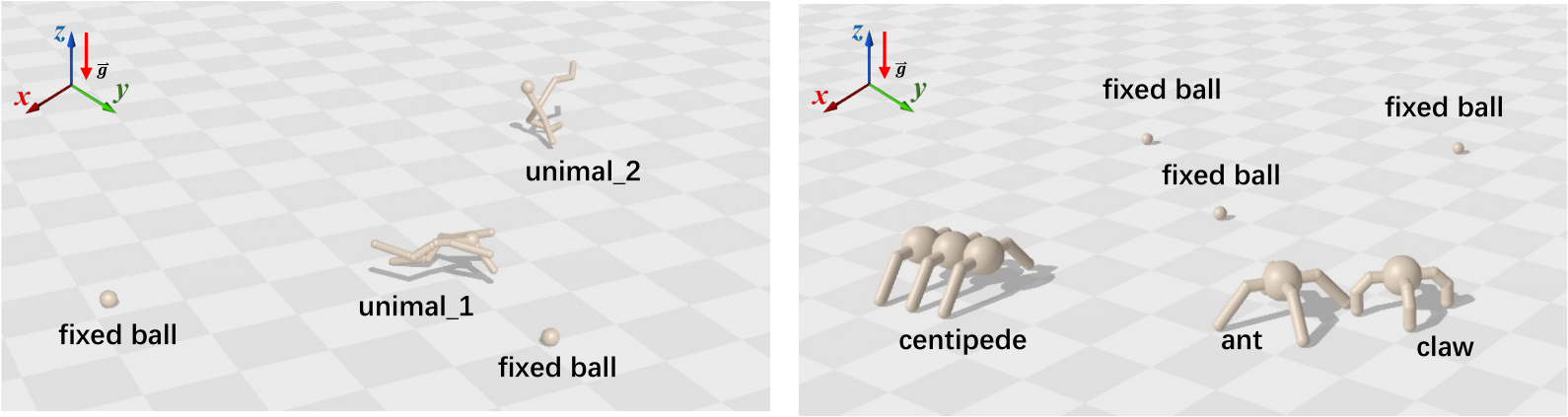}
\caption{Team Reach Environments. Here, displaying \texttt{2\_unimals} and \texttt{3\_ant\_claw\_centipede} environments.}
\label{fig:reach_demo}
\vspace{-.13in}
\end{figure}

\paragraph{Team Reach}
We expand the ``Reach'' task from a single-agent challenge to a collaborative ``Team Reach'' task, as shown in \cref{fig:reach_demo}. 
\textbf{1. Initial Conditions.} Entities set $\Omega$ include \( N \) agents and \( M \) fixed balls, $N \geq M$. Within an area of radius \( R \), we randomly position \( N \) agents and \( M \) fixed balls, also setting the initial orientations of the agents randomly.
Specific details for \( N \), \( M \), \( R \), and the agents' morphology are provided in \cref{tab:environments}.
\textbf{2. Termination.} The goal of this task is for \( M \) fixed balls to be \emph{simultaneously} occupied by agents. An episode is considered successfully completed and terminates once this condition is fulfilled. 
\textbf{3. Reward.}
The designed reward structure comprises four components:
\textbf{a. Success Bonus:} A significant sparse reward of 10,000 is awarded.
\textbf{b. Distance Reward:} A dense reward to incentivize the achievement of the task objective. It is computed as \( 5 \times \sum_{j=1}^{M} \exp(-\text{dist}_j) \), where \(\text{dist}_j\) represents the distance from the nearest agent to ball \( j \).
\textbf{c. Moving Reward:} Designed to motivate agents to move closer to any ball, it is quantified as \( 0.2 \times \sum_{i=1}^{N}\sum_{j=1}^{M} \max(\vec{\vv}_{i} \cdot (\vec{\vp}_{j} -  \vec{\vp}_{i}), 0) \).
\textbf{d. Control Cost:} This penalty discourages agents from executing excessively large actions, and is calculated as \( -0.2 \times \sum_{i=1}^{N}\sqrt{\sum_{k=1}^{K_i}(a_{i,k})^2} \).
Thus, the shared team reward is equal to the sum of the aforementioned rewards.

\begin{figure}[ht]
\centering
\includegraphics[width=0.8\linewidth]{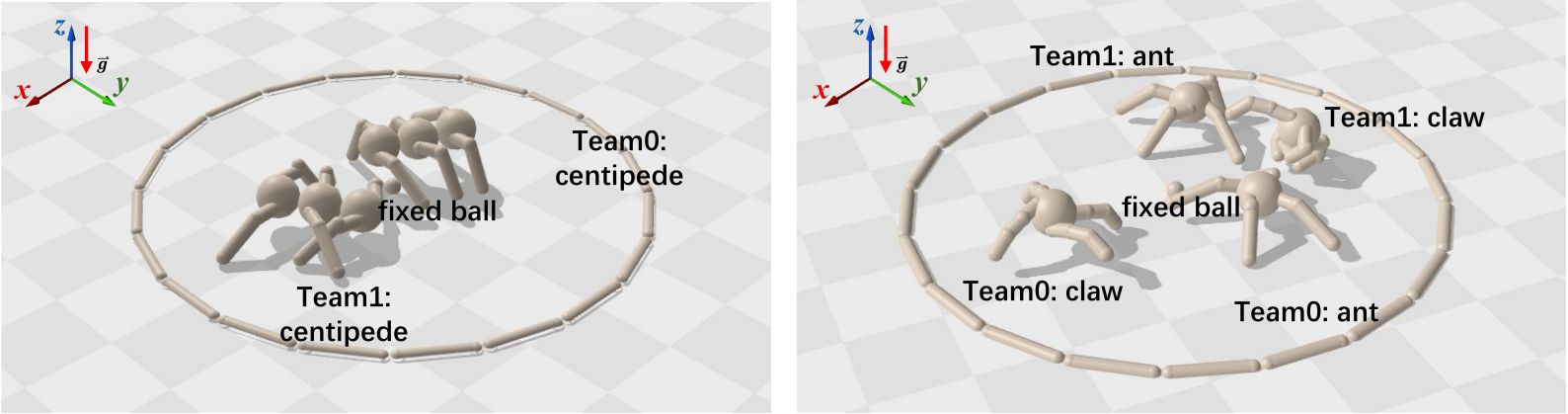}
\caption{Team Sumo Environments. Here, displaying \texttt{1\_centipede} and \texttt{2\_ant\_claw} environments.}
\label{fig:sumo_demo}
\vspace{-.13in}
\end{figure}

\paragraph{Team Sumo}
We evolve the ``Sumo'' task from a purely competitive challenge to a mixed cooperative-competitive ``Team Sumo'' task, as shown in \cref{fig:sumo_demo}. 
\textbf{1. Initial Conditions.} Entities set $\Omega$ comprise one fixed ball, \( N \) agents forming Team 1, and another \( M \) agents constituting Team 2.
The sumo arena, a circle with radius \( R \), has its center marked by the fixed ball.
Around this fixed ball, within a radius of \( R-1 \), we randomly position  \( N \) agents from Team 1 and another \( M \) from Team 2, ensuring each agent's orientation is also randomized. 
Specific details for \( N \), \( M \), \( R \), and the agents' morphology are provided in \cref{tab:environments}.
\textbf{2. Termination.} The objective is for either Team 1 or Team 2 to win by having an opposing team's agent disqualified, which occurs if it exceeds a distance \( R \) from the fixed ball. The team with the disqualified agent loses, triggering the termination of the episode. 
\textbf{3. Reward.}
The reward designed for Team 1 (for Team 2, simply swap $N$ and $M$) is divided into five parts:
\textbf{a. Win Bonus:} A sparse reward of 1000 for achieving the win condition.
\textbf{b. Lose Bonus:} A sparse penalty of -1000 for the losing condition.
\textbf{c. Distance Reward:} A dense reward to encourage achieving the task's objective, computed as \( 5 \times \exp(\text{dist} - R) \), where \(\text{dist}\) denotes the distance of the farthest agent from the opposing team to the fixed ball.
\textbf{d. Moving Reward:} This motivates agents to move closer to any member of the opposing team, calculated as \( 5 \times \sum_{i=1}^{N}\sum_{j=1}^{M} \max(\vec{\vv}_{i} \cdot (\vec{\vp}_{j} -  \vec{\vp}_{i}), 0) \).
\textbf{e. Control Cost:} A penalty for excessively large actions by agents, quantified as \( -0.1 \times \sum_{i=1}^{N}\sqrt{\sum_{k=1}^{K_i}(a_{i,k})^2} \).
In line with the Exploration Curriculum suggested by \citet{bansal2018emergent}, the shared team reward is calculated as \( \alpha \times \text{dense reward} + (1 - \alpha) \times \text{sparse reward} \), where \( \alpha \) is a linear annealing factor. Agents train with the dense reward for 25\% of the training epochs.

\begin{table}[ht]
\centering
\caption{Full list of environments used in this work.}
\begin{tabular}{c c c c c c}
\toprule
\textbf{Benchmark} & \textbf{Environment} &  \textbf{Morphology}  & $N$ & $M$ & $R$ (meter) \\
\midrule
\multirow{6}{*}{Team Reach } & \texttt{1\_ant} & ant & 1 & 1 & 5 \\ 
& \texttt{1\_centipede} & centipede & 1 & 1  & 5\\ 
& \texttt{2\_ants} & ant & 2 & 2 &  3\\ 
& \texttt{2\_ant\_claw} & ant, claw & 2 & 2 &   5\\ 
& \texttt{2\_unimals} & unimal\_1, unimal\_2 & 2 & 2 & 5 \\
& \texttt{3\_ant\_claw\_centipede} & ant, claw, centipede & 3 & 3 &  5 \\ 
\midrule
\multirow{4}{*}{Team Sumo} & \texttt{1\_centipede} \, vs \, \texttt{1\_centipede} & centipede & 1 & 1 &  3\\ 
& \texttt{2\_ants} \, vs \, \texttt{2\_ants} & ant & 2 & 2 &   3\\ 
& \texttt{2\_ant\_claw} \, vs \, \texttt{2\_ant\_claw} & ant, claw & 2 & 2 & 3 \\ 
& \texttt{3\_ants} \, vs \, \texttt{3\_ants} & ant & 3 & 3 &   4\\ 
\bottomrule
\end{tabular}
\label{tab:environments}
\end{table}

\subsection{Baselines}\label{sec:baselines}
We compare our method SHNN against mainstream neural networks, particularly MLP as utilized in ~\cite{furuta2023asystem}, and a variant employing heading normalization technique, denoted as MLP+HN~\cite{chen2023sgrl}.

\paragraph{MLP} Our implementation of MLP consists of three linear layers. The middle two layers are activated by ReLU activation functions. The input layer is designed to accommodate the size of the observation space, while the output layer matches the dimensionality of the action space. The intermediate layers facilitate the non-linear transformation of input features. For specific hyperparameters, please refer to ~\cref{tab:nethyp}.

\paragraph{MLP+HN}  Heading Normalization technique has been widely utilized in the 3D RL literature. For instance, in morphology control, the presence of gravity allows for the normalization of state and action spaces in the heading (yaw) direction, as demonstrated in a recent work~\cite{won2020scalable,won2022physics,chen2023sgrl}. This heading normalization (HN) technique transforms the global coordinate frame into the LRF, enabling the input geometric information to be mapped to a rotation- and translation-invariant representation.  

Specifically, we acquire the quaternion of the root body from the simulation environment and calculate the heading angle to construct the rotation matrix $\mO_i$. This matrix is then multiplied by $\vec{\mZ}_{i,k}$ to transform it into an invariant representation. Finally, all invariant representations are input into the MLP. It is noteworthy that \citet{chen2023sgrl} discusses the limitations of this technique in their appendix.

\begin{algorithm}[ht]
  \caption{Greedy Bipartite Matching for Task Assignment}\label{alg:greedy_bipartite_ocp}
\begin{algorithmic}
    \STATE \textbf{Input}: $\vec{\vp}_{i,1}, i \in \Omega$, \( \gN \) \COMMENT{agents (Team0) set}, \( \gM  \) \COMMENT{objects (Team1) set}
    \STATE \textbf{Output}: A local entity-level graph \( \gG = (\gV, \gE) \)
    
    \STATE \emph{Initialize the graph} \( \gG \) with vertices \( \gV \gets \Omega \) and edges \( \gE \gets \emptyset \)
    
    \STATE \emph{Initialize the assignment labels}  \( \mathcal{C} \gets \emptyset \)
    
    \STATE \emph{Initialize distance matrix} \( D \) between \( \gN  \) and \( \gM  \) 
    \FOR{each \( i \in \gN \) and \( j \in \gM \)}
       \STATE \( D[i,j] \gets \|\vec{\vp}_{j,1} - \vec{\vp}_{i,1}\|_2 \)
    \ENDFOR
    
    \FOR{each \( j \in \gM \)}
        \STATE \( i = \argmin_{k \in \gN} D[k,j] \)
        \STATE \( D[i,:] \gets \infty \)  \COMMENT{Prevent re-matching}
        \STATE Add \( (i, j) \) to \( \gE \): \( \mathcal{C}(i) = i, \mathcal{C}(j) = i \)
    \ENDFOR
    
    \STATE \textbf{return} A local entity-level graph \( \gG = (\gV, \gE) \)
\end{algorithmic}
\end{algorithm}

\subsection{Task Assignment}

\paragraph{Greedy Bipartite Matching Algorithm}
The Greedy algorithm is a fundamental optimization algorithm, this part is about the process of applying the greedy matching algorithm to allocate a fixed ball (or an opposing agent) to each agent in the environments. \Cref{alg:greedy_bipartite_ocp} provides a pseudo-code implementation of the Greedy Bipartite Matching algorithm.

\subsection{Additional Ablations}

\paragraph{Additional Ablations on Assignment} \label{sec:add_cluster}
In the Team Sumo environments, task assignment yields limited improvement.
This could be attributed to the environments' complex mixture of cooperative and competitive dynamics, which are challenging to effectively decouple using bipartite matching methods, as demonstrated in ~\cref{fig:sumo_ben_bipair}.

\begin{figure}[htpb!]
\centering
\includegraphics[width=0.6\linewidth]{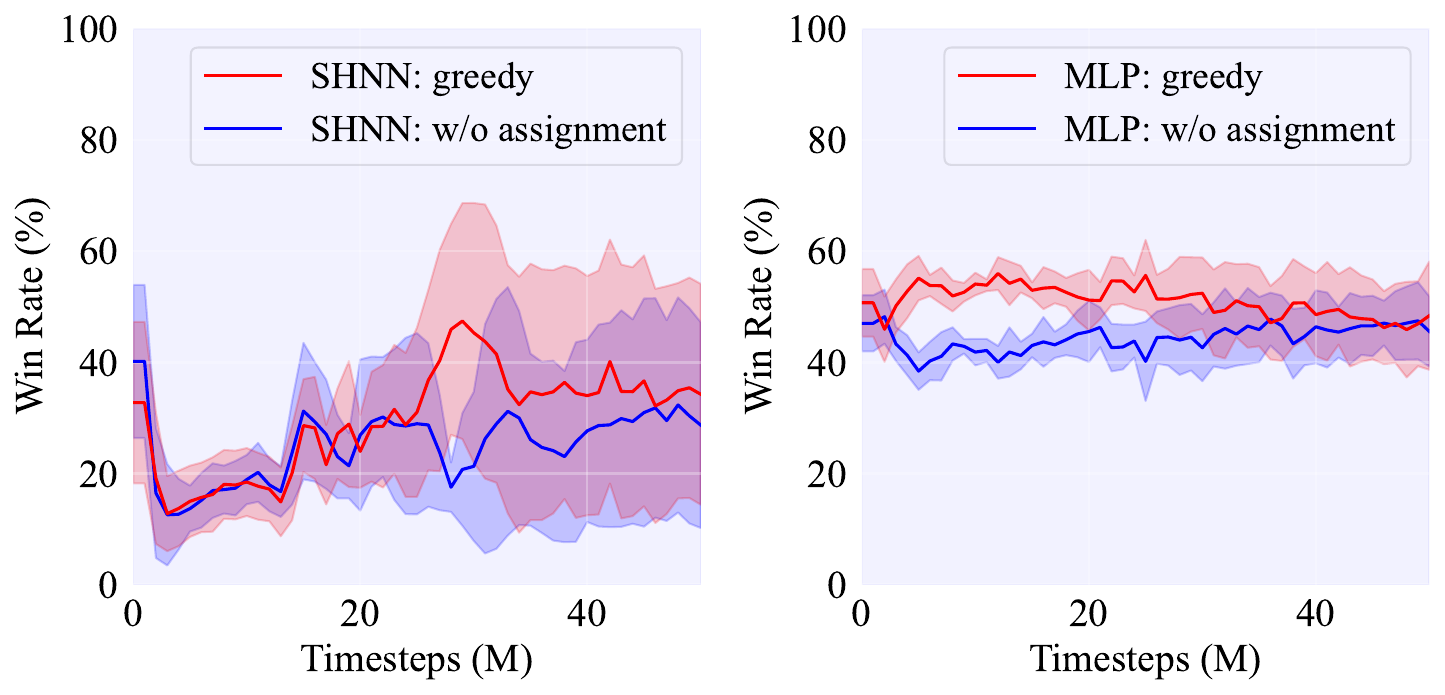}
\vspace{-.12in}
\caption{Additional Ablations on Assignment. Training and Evaluation Curves in \texttt{2\_ants} Team Sumo Environment. 
}
\label{fig:sumo_ben_bipair}
\vspace{-.13in}
\end{figure}

\paragraph{Additional Ablations on Equivariance} 
In the Team Sumo environments, task assignment leads to the formation of local graphs composed of triplets of agents from both sides and a fixed ball, making it challenging to determine the specific goal orientation learned by equivariant networks. Consequently, our method was primarily compared against HNN+HN. The experimental outcomes align with the main text findings, as demonstrated in ~\cref{fig:sumo_ablation}.

\begin{figure}[htpb!]
\centering
\includegraphics[width=0.7\linewidth]{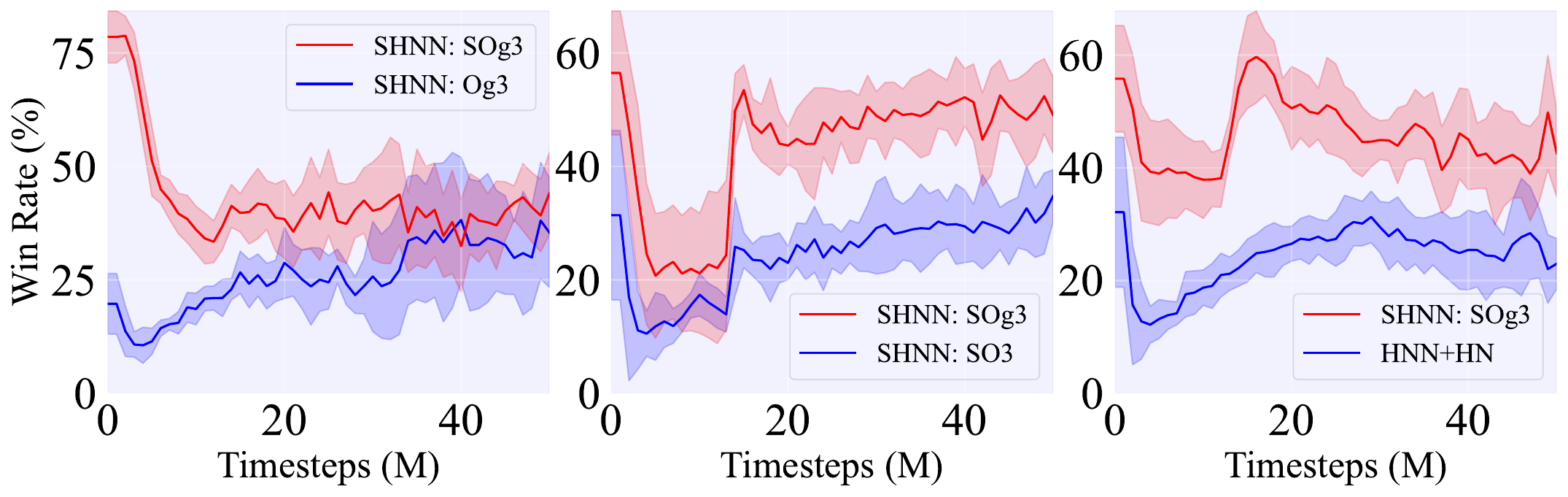}
\vspace{-.12in}
\caption{Additional Ablations on Equivariance. Training and Evaluation Curves in \texttt{2\_ants} Team Sumo Environment.
}
\label{fig:sumo_ablation}
\vspace{-.13in}
\end{figure}

\paragraph{Importance of Local Symmetry} 
By comparing the red and blue lines in the left plot of \cref{fig:sumo_ben_bipair} and in the right plot of \Cref{fig:sumo_ablation}, it becomes evident that in the Team Sumo environments, the impact of assignment is not as significant as that of equivariance. 
Therefore, careful design of assignment strategies is crucial to harness the advantages offered by equivariance effectively.

\begin{figure}[htpb!]
\centering
\includegraphics[width=\linewidth]{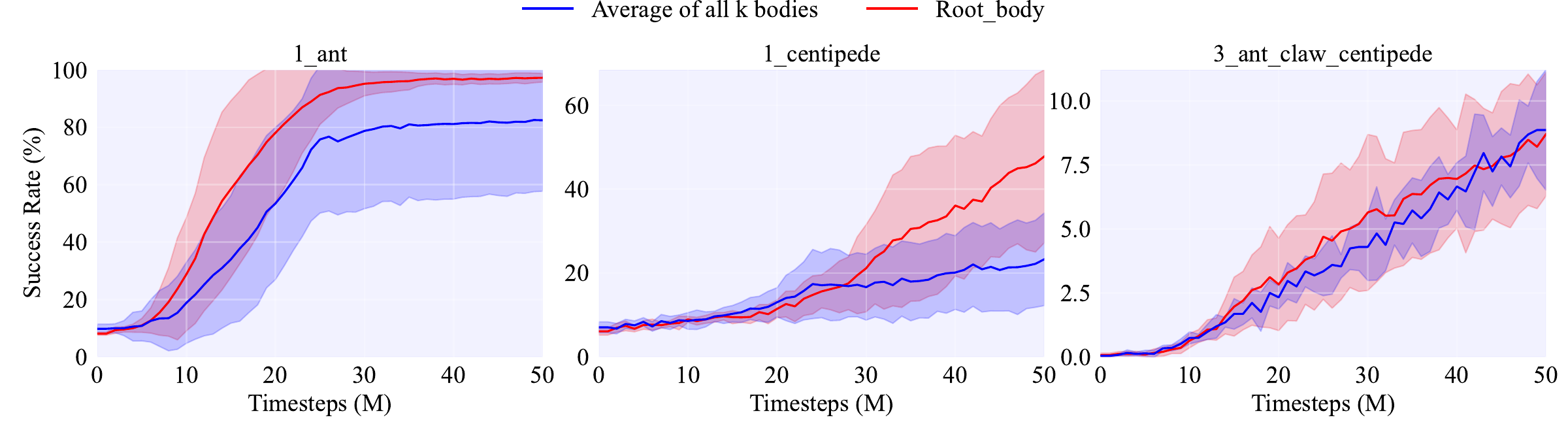}
\vspace{-.12in}
\caption{Additional Ablations on Entity Abstraction. Training and Evaluation Curves in \texttt{1\_ant},  \texttt{1\_centipede}, and \texttt{3\_ant\_claw\_centipede} Team Reach Environments.
}
\label{fig:ablation_translation}
\vspace{-.13in}
\end{figure}

\paragraph{Ablation Study on Entity Abstraction} 
The additional ablation study examines the effects of using different abstraction for each entity. Our hierarchical structure dictates a representative abstraction for each entity, and within RL environments, the reward is often contingent on the root body's state. Utilizing the first (root) body as a representative aligns with the common computational practices in RL, focusing on the pivotal elements that influence agent behavior and reward structures. 
Moreover, representing entity-level information through the mean of all K bodies is a viable alternative, offering a more comprehensive local dynamic. The ablation study examines the effects of using the root body versus the average of all K bodies:
\textbf{Root body}: $\vec{\bm{Z}}_i$ is assigned as $\vec{\bm{Z}}_{i,1}$,  $\bm{h}_i$ is set as $[\bm{h}_{i,1}, \vec{\bm{p}}_{i,1}^z]$, $\vec{\bm{Z}}_{ij} = [(\vec{\bm{p}}_{j,1} - \vec{\bm{p}}_{i,1}),\vec{\bm{Z}}_{i},\vec{\bm{Z}}_{j}]$, and $\bm{h}_{ij} = [\|\vec{\bm{p}}_{j,1} - \vec{\bm{p}}_{i,1}\|_2, \bm{h}_{i}, \bm{h}_{j}]$;
\textbf{Average of all $K$ bodies}: $\vec{\bm{Z}}_i$ is assigned as $\frac{1}{K_i}\sum_{k=1}^{K_i}\vec{\bm{Z}}_{i,k}$,  $\bm{h}_i$ is set as $[\frac{1}{K_i}\sum_{k=1}^{K_i}\bm{h}_{i,k}, \frac{1}{K_i}\sum_{k=1}^{K_i}\vec{\bm{p}}_{i,k}^z]$,$\vec{\bm{Z}}_{ij} = [(\frac{1}{K_j}\sum_{k=1}^{K_j}\vec{\bm{p}}_{j,k} -\frac{1}{K_i}\sum_{k=1}^{K_i} \vec{\bm{p}}_{i,k}),\vec{\bm{Z}}_{i},\vec{\bm{Z}}_{j}]$, and $\bm{h}_{ij} = [\|\frac{1}{K_j}\sum_{k=1}^{K_j}\vec{\bm{p}}_{j,k} - \frac{1}{K_i}\sum_{k=1}^{K_i}\vec{\bm{p}}_{i,k}\|_2, \bm{h}_{i}, \bm{h}_{j}]$. 
The empirical results in \Cref{fig:ablation_translation}, particularly for the asymmetric \texttt{1\_centipede} scenario, reveal a significant performance gap, clearly demonstrating the superiority of using the root body's state over the average of all K bodies.

\begin{figure}[htpb!]
\centering
\includegraphics[width=0.8\linewidth]{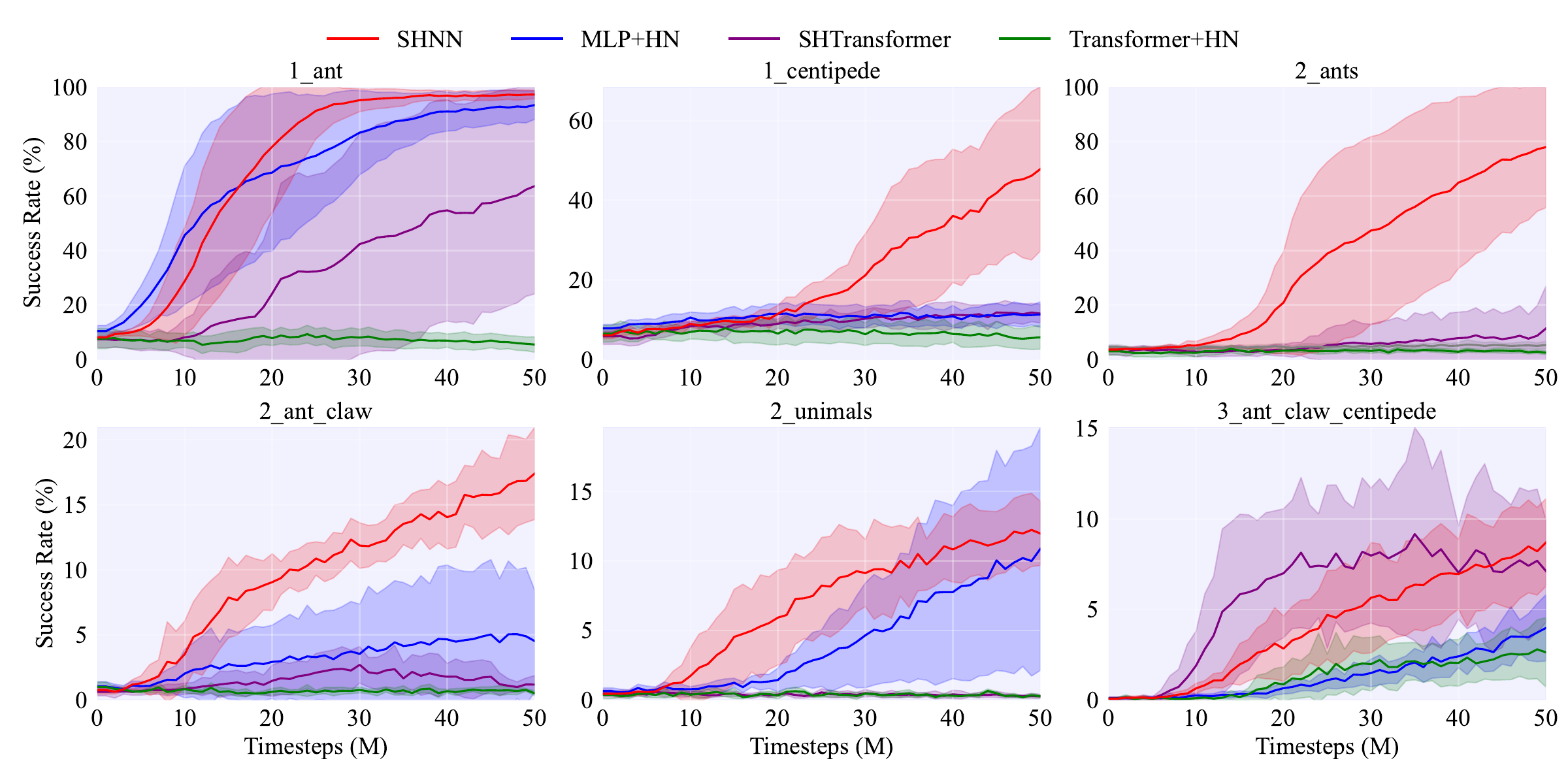}
\vspace{-.12in}
\caption{Evaluations on Transformer Architecture. Training and Evaluation Curves in Team Reach Environments. 
}
\label{fig:transformer}
\vspace{-.13in}
\end{figure}

\subsection{Analyses of Morphology-shared Policy}
\label{sec:sgnn}
We provide an extended example of our methodology applied to morphology tasks.
Previous works have achieved generalization across agents with different morphologies using morphology-aware Graph Neural Networks \cite{wang2018nervenet,huang2020one}.

Our variant, referred to as SHGNN, replaces the body-level MLP with body-level message passing. This adaptation enables the learning of a shared policy network across different agents.

First, we utilize the morphology topology information of each agent \(i\) to construct a body-level inner-entity graph, denoted as \( \gG_i = (\gV_i, \gE_i) \).
For each body, \( k \in \gV_i\), input node features are initialized using the body's state. Specifically, \( \vec{\mZ}_k \) is assigned as \( \vec{\mZ}_{i,k} \), and \( \vh_k \) is set as \( [\vh_{i,k}, \vh'_i, \vec{\vp}_{i,k}^z] \), where [ ] is the stack along the last dimension and \( \vec{\vp}_{i,k}^z \) represents the projection of the coordinate \( \vec{\vp}_{i,k} \) onto the \( z \)-axis. 
In a body-level overview, we denote our body-level message passing as the function $\varphi_b$ that updates each body's node features given the input node features of all body and graph connectivity:
\begin{align}
\label{eq:somp1}
    \{(\vec{\mZ}'_k, \vh'_k)\}_{k=1}^{K_i} &= \varphi_b\left(\{(\vec{\mZ}_k,\vh_k)\}_{k=1}^{K_i}, \gE_i \right).
\end{align}
Notably, the unfolding of $\varphi_b$ is similar to that of $\varphi_o$, and will not be elaborated further here.

For each agent \( i \), the invariant actor policy \( \pi_{\theta_i} \) is defined as
\begin{align}
    \pi_{\theta_i} = \{\pi_{\theta_{i,k}}\}_{k=2}^{K_i} = \{\sigma_{\pi}(\mO_i^{\top}\vec{\mZ}'_k, \vh'_k)\}_{k=2}^{K_i},
\end{align}
where \( \sigma_{\pi} \) is a linear layer with bias. Here, \( \pi_{\theta_i} \in \sR^{2 \times (K_i-1)} \) represents the location and scale parameters of a Normal Tanh Distribution for the \( (K_i-1) \) actuators of agent \( i \). Each actuator samples its corresponding torque \( a_{i,k} \in [-1,1] \) from this distribution.
The invariant critic value-function \( V_{\phi} \) remains as described in the main text.
Within the same team, different agents share the weight parameters of the body-level message passing.

\begin{figure}[htpb!]
\centering
\includegraphics[width=0.5\linewidth]{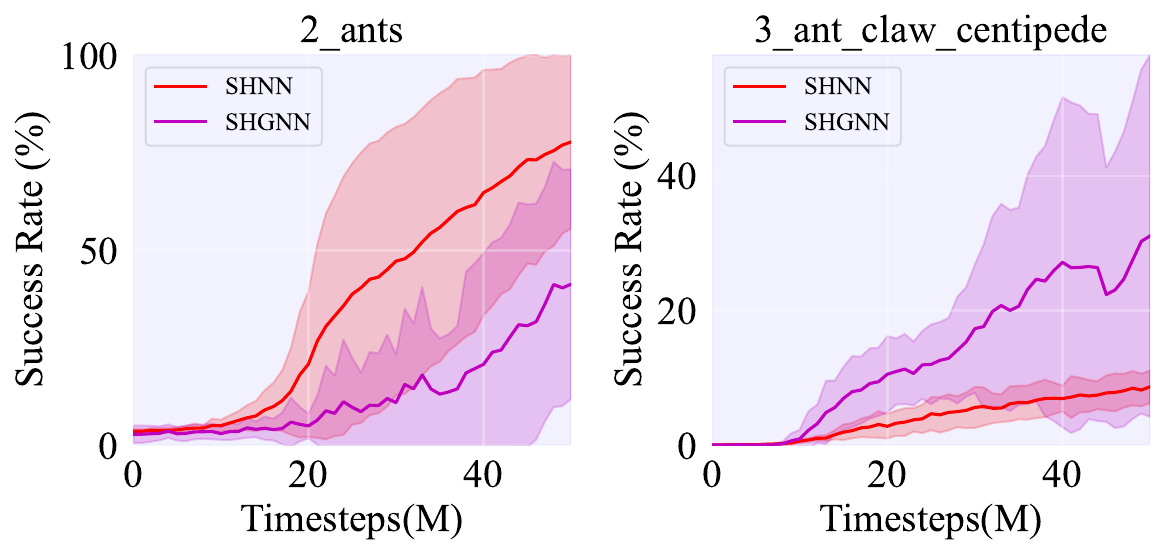}
\vspace{-.12in}
\caption{Analysis of Morphology-shared Policy. Training and Evaluation Curves in Team Reach Environments. 
}
\label{fig:low_graph}
\vspace{-.13in}
\end{figure}
As observed in \cref{fig:low_graph}, while body-level message passing struggles to learn effective control strategies for a single morphology (e.g., \texttt{2\_ants}), it interestingly excels in more complex, multi-morphological environments like \texttt{3\_ant\_claw\_centipede}. In such scenarios, knowledge sharing across different morphologies leads to mutual enhancement, significantly outperforming the body-level MLP approach. However, as indicated in~\cref{tab:model_comparison}, the implementation of message passing considerably slows down the training process.

\subsection{Model Comparison: Parameters and Training Time}

\cref{tab:model_comparison} compares the parameters and Training Wall Times of our model with several model variants in the \texttt{3\_ant\_claw\_centipede} Team Reach Environments.
Here, total timesteps are 50M. 
Since Transformer and GNN can share parameters across different morphologies, they have fewer parameters.

\begin{table}[ht]
\centering
\caption{Comparison of Models in Terms of Parameters and Training Wall Time in \texttt{3\_ant\_claw\_centipede} Team Reach Environments.  }
\label{tab:model_comparison}
\begin{tabular}{lcc}
\toprule
\textbf{Model} & \textbf{Parameters} (M) & \textbf{Training Wall Time} (h) \\
\midrule
MLP+HN & \textbf{1.759} & 0.299$\pm$0.003 \\
Transformer+HN & 0.416 & 1.545$\pm${0.001} \\
SHNN & 0.772 & 1.302$\pm${0.012} \\
SHTransformer & 0.711 &  2.487$\pm${0.014} \\
SHGNN & 0.613 &  \textbf{5.538}$\pm${1.411} \\
\bottomrule
\end{tabular}
\end{table}

\subsection{Equivariance Test}\label{sec:eq_test}
We conduct an experiment, as depicted in \cref{fig:eq_test_demo}, to evaluate the rotational generalization of both the baselines and our method.
Training and evaluation are conducted in the fixed initial conditions of the \texttt{1\_centipede} and \texttt{2\_ants} Team Reach environments. Additionally, we conduct an evaluation with a 180° rotation of the entire scene. 
The results presented in \cref{tab:eq_test}, with detailed curves provided in \cref{fig:eq_test}, illustrate that both our SHNN method and the MLP+HN approach exhibit stable performance pre- and post-rotation.
Conversely, MLP demonstrates rotational generalization for symmetric morphologies, such as ants, yet entirely lacks this capability with asymmetric morphologies like centipedes, evidenced by the underline in \cref{tab:eq_test} and the green line in the first plot of \cref{fig:eq_test}. 
These results empirically validate that both our SHNN method and the HN approach are rotation equivariance, which can robustly generalize to unseen rotation transformations.

\begin{figure}[ht]
\centering
\includegraphics[width=0.7\linewidth]{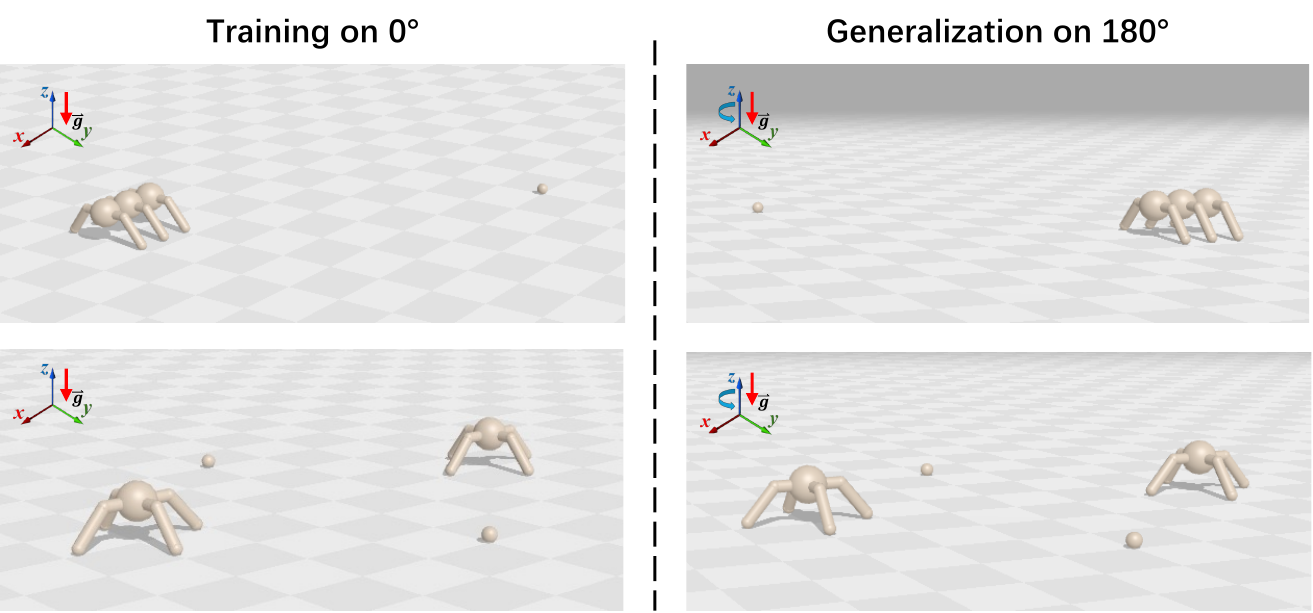}
\caption{Equivariance Test Scenarios. Training are conducted in the fixed initial conditions, a generalization evaluation is conducted with a 180° rotation of the entire scene.}
\label{fig:eq_test_demo}
\vspace{-.13in}
\end{figure}

\begin{table}[ht]
    \centering
    \caption{Equivariance Test. We report Success Rate (\%) on the final step in Team Reach Environments.  }
    \normalsize
    \label{tab:eq_test}
    \begin{tabular}{ccccc}
    \toprule
     \multirow{2}{*}{\textbf{Methods}} & \multicolumn{2}{c}{\texttt{1\_centipede}} & \multicolumn{2}{c}{\texttt{2\_ants}} \\
     & $0^\circ$ & $180^\circ$ & $0^\circ$ & $180^\circ$ \\
     \midrule
     MLP & $\textbf{41.55}\pm19.70$ & \underline{$0.18\pm0.22$} & $44.00\pm12.90$ & 
     \underline{$40.03\pm12.03$} \\
     MLP+HN & $38.02\pm24.98$ & $37.92\pm25.91$ & $39.19\pm12.85$ & $38.60\pm12.53$ \\
     SHNN & $40.23\pm22.28$ & $\textbf{41.39}\pm21.89$ & $\textbf{87.78}\pm6.59$ & $\textbf{87.84}\pm6.42$ \\
    \bottomrule
    \end{tabular}
    \vspace{-.1in}
\end{table}

\begin{figure}[ht]
\centering
\includegraphics[width=0.7\linewidth]{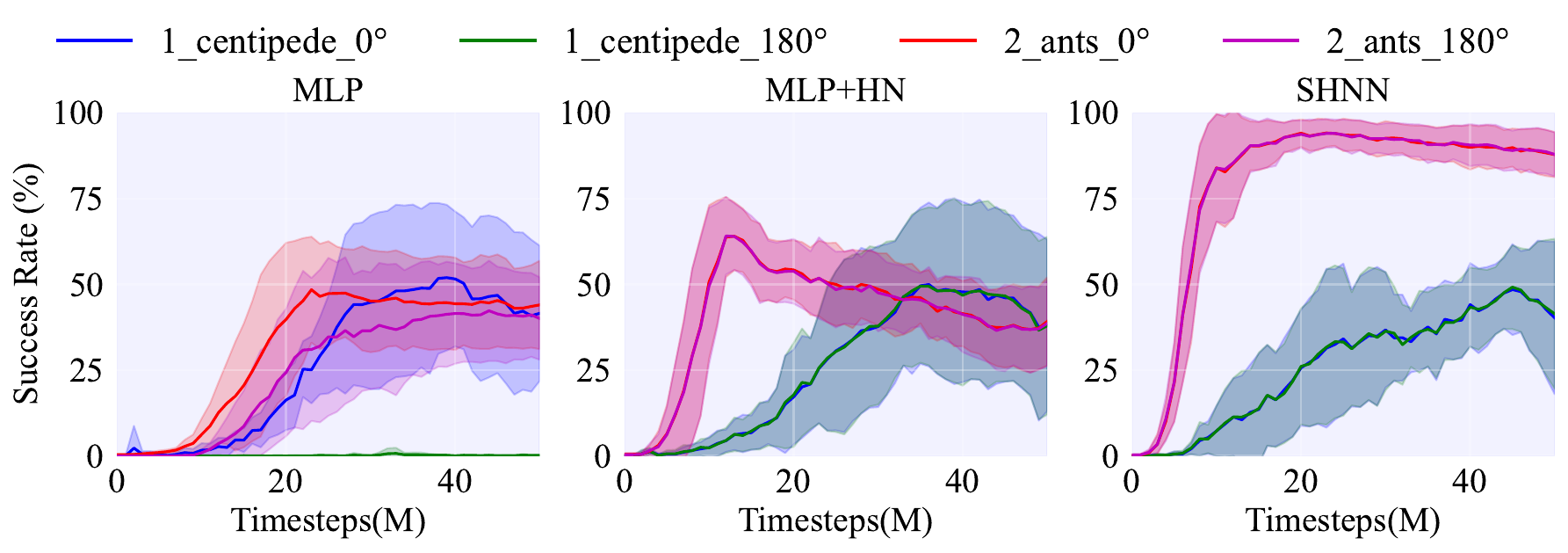}
\vspace{-.13in}
\caption{Equivariance Test. Training and Evaluation Curves in Team Reach Environments.
MLP+HN and SHNN, by explicitly addressing local transformations, exhibit overlapping curves in equivariance tests.
}
\label{fig:eq_test}
\vspace{-.13in}
\end{figure}

\subsection{Hyperparameters}\label{sec:imple}
\label{sec:hypers} 
\cref{tab:mappohyp} and  \cref{tab:nethyp} provide the hyperparameters needed to replicate our experiments. 
Code and Environments are available on our project page: \href{https://alpc91.github.io/SMERL/}{https://alpc91.github.io/SMERL/}.

\vspace{-.2in}
\begin{table}[ht!]
\centering
\scriptsize
\begin{minipage}{.45\linewidth}
\centering
\caption{Hyperparameters of MAPPO.}
\begin{tabular}{lc}
\toprule
\textbf{Hyperparameter} & \textbf{Value} \\
\midrule
total\_timesteps & Team Reach: 50M / Team Sumo: 20M \\
eval\_frequency & 50 \\
num\_envs & 2048 \\
action\_repeat & 1 \\
batch\_size & Team Reach: 1024 / Team Sumo: 128 \\
reward\_scaling & 1.0 \\
episode\_length & 1000 \\
entropy\_cost & 1e-2 \\
unroll\_length & 5 \\
discounting & 0.97 \\
learning\_rate & 3e-4 \\
num\_minibatches & 32 \\
num\_update\_epochs  & 4 \\
gradient\_clipping  & 0.1 \\
normalize\_observations & True \\
\bottomrule
\end{tabular}
\label{tab:mappohyp}
\end{minipage}
\hfill
\begin{minipage}{.45\linewidth}
\centering
\caption{Hyperparameters of Network.}
\begin{tabular}{lcc}
\toprule
\textbf{Module} & \textbf{Hyperparameters} & \textbf{Value} \\
\midrule
\multirow{4}{*}{MLP} & hidden\_dim & 256 \\
 & output\_dim & $\pi$: $2 \times (K_i-1)$ / $V$: 1 \\
 & \# linear\_layers & 3 \\
 & activation & relu \\
\midrule
\multirow{5}{*}{Message Passing} & hidden\_dim & 64 \\
 & vector\_dim & 32 \\
 & \# MLP\_layers & 2 \\
 & activation & relu \\
 & propagation\_steps & 2 \\
 \midrule
\multirow{8}{*}{Transformer} & model\_dim & 128 \\
 & feedforward\_dim & 256 \\
 & \# layers & 3 \\
 & \# heads & 2 \\
 & activation & relu \\
 & transformer\_norm & LayerNorm \\
 & condition\_decoder & True \\
 & positional\_encoding & False \\
\bottomrule
\end{tabular}
\label{tab:nethyp}
\end{minipage}
\end{table}

\end{document}